\RequirePackage{fix-cm}
\documentclass[twocolumn]{svjour3}          % twocolumn
\smartqed  % flush right qed marks, e.g. at end of proof
\usepackage{graphicx}
\usepackage{times}
\usepackage{graphicx}
\usepackage{amsmath}
\usepackage{amssymb}
\usepackage{graphicx}
\usepackage{amssymb}
\usepackage{amsfonts}
\usepackage{mathrsfs}
\usepackage{algorithm}
\usepackage{algorithmic}
\usepackage{array}
\usepackage{booktabs}
\usepackage{float}
\usepackage[table]{xcolor}\usepackage{setspace}
\usepackage{bm}
\usepackage{multirow}%\usepackage{bm}
\usepackage{makecell}
\allowdisplaybreaks[2]
\renewcommand{\algorithmicrequire}{\textbf{Input:}}
\renewcommand{\algorithmicensure}{\textbf{Output:}}
\newenvironment{shrinkeq}[1]
{ \bgroup
  \addtolength\abovedisplayshortskip{#1}
  \addtolength\abovedisplayskip{#1}
  \addtolength\belowdisplayshortskip{#1}
  \addtolength\belowdisplayskip{#1}}
{\egroup\ignorespacesafterend}
\DeclareMathOperator*{\argmin}{argmin}
\usepackage[pagebackref=false,breaklinks=true,letterpaper=true,colorlinks,bookmarks=false]{hyperref}
\def\X{{\mathbf{X}}}
\def\U{{\mathbf{U}}}
\def\V{{\mathbf{V}}}

\def\G{{\mathbf{G}}}
\def\E{{\mathbf{E}}}
\def\Z{{\mathbf{Z}}}

\begin{document}

\title{ Enhanced 3DTV Regularization and Its Applications on Hyper-spectral Image Denoising and Compressed Sensing%\thanks{Grants or other notes
}
%Gradient Correlation Sparsity prior for Hyper-spectral Images and Its Applications
% 孟老师建议作者： Jiangjun Peng, Qi Xie, Qian Zhao, Yao Wang, Deyu Meng
\author{Jiangjun Peng \and
        Qi Xie  \and
        Qian Zhao \and
        Yao Wang \and
        Yee Leung \and
        Deyu Meng
}
%\author{First author        \and
%        Second author

\institute{J. Peng, Q. Xie, Q. Zhao, Y. Wang, D. Meng (corresponding author) \at
Institute for Information and System
Sciences and Ministry of Education Key Lab of Intelligent Networks and
Network Security, Xian Jiaotong University, Xi��an, Shaan��xi, 710049 P.R.
China\\
              \email{dymeng@mail.xjtu.edu.cn}           %  \\
           \and
           Y. Leung \at
              Department of Geography and Resource Management
and Institute of Future Cities, The Chinese University of Hong Kong, Shatin,
Hong Kong.
           \and
           Jiangjun Peng and Qi Xie made equal contributions to this work.
}


\maketitle

\begin{abstract}
The 3-D total variation (3DTV) is a powerful regularization term, which encodes the local smoothness prior structure underlying a hyper-spectral image (HSI), for general HSI processing tasks. This term is calculated by assuming identical and independent sparsity structures on all bands of gradient maps calculated along spatial and spectral HSI modes. This, however, is always largely deviated from the real cases, where the gradient maps are generally with different while correlated sparsity structures across all their bands. Such deviation tends to hamper the performance of the related method by adopting such prior term. To this issue, this paper proposes an enhanced 3DTV (E-3DTV) regularization term beyond conventional 3DTV. Instead of imposing sparsity on gradient maps themselves, the new term calculated sparsity on the subspace bases on the gradient maps along their bands, which naturally encode the correlation and difference across these bands, and more faithfully reflect the insightful configurations of an HSI. The E-3DTV term can easily replace the previous 3DTV term and be embedded into an HSI processing model to ameliorate its performance. The superiority of the proposed methods is substantiated by extensive experiments on two typical related tasks: HSI denoising and compressed sensing, as compared with state-of-the-arts designed for both tasks.

\end{abstract}
%% ===========================================================================
\section{Introduction}
\label{intro}

The radiance of a real scene is distributed across a wide range of spectral bands.
A hyperspectral image (HSI) consists of various intensities that represent the integrals of the radiance captured by sensors over hundreds of discrete bands. As compared with traditional image systems, HSI facilitates delivering more faithful representation for real scenes, and thus tends to be better performed on various computer vision tasks, such as classification \cite{Zhao2016High}, super-resolution \cite{Dong2016HSI}, compressed sensing \cite{Zhang2015HSI,Mart2015HYCA}, and mineral exploration \cite{W2013S,Goetz2009}.

In real cases, however, an HSI is always significantly corrupted by noises that are generally conducted by sensor sensitivity, photon effects, light condition and/or calibration error \cite{Goetz2009}. The HSI denoising problem is thus a critical issue and the resolving of the problem could greatly ameliorate the performance of the subsequent HSI processing tasks. Besides,  the data need to be transmitted to ground station to deal with, but an HSI practically collected from an airborne sensor or a satellite is always with hundreds of image bands, which makes the HSI requiring a large storage space. This will conduct severe issues of low efficiency and high cost on transmitting them to the ground stations. Thus it is necessary to design effective techniques on HSI compressed sensing to satisfy the on-time transmit.

\begin{figure*}
\centering
\includegraphics[width=1\linewidth]{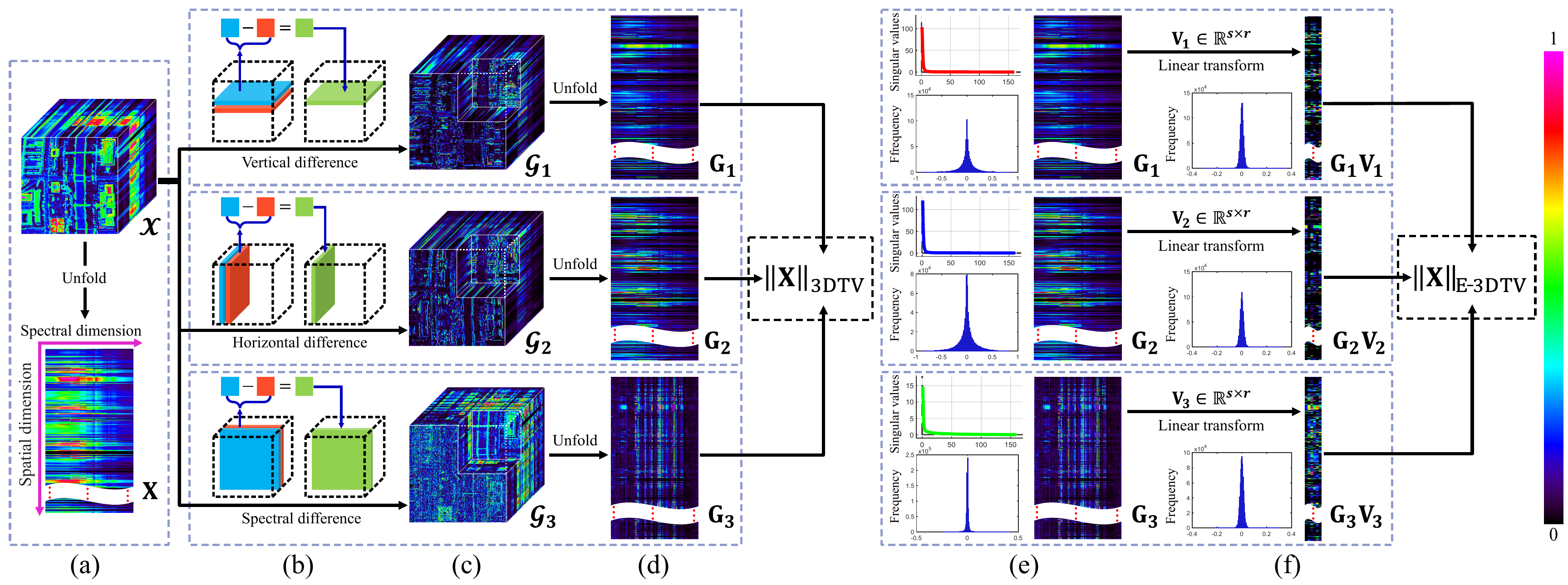}
\caption{Illustration of the 3DTV and the proposed E-3DTV regularization terms. (a) An example of real clean HSI, represented as a tensor $\mathcal{X}\in\mathbb{R}^{200\times200\times160}$ (upper). It can also be represented as the unfolding matrix of $\mathcal{X}$ along the spectrum mode, $\X\in\mathbb{R}^{(200\cdot200)\times160}$ (lower); (b) Illustrations of the difference operators along the spatial height, width and spectral modes, respectively; (c) The gradient maps of $\mathcal{X}$ in spatial height, width and spectrum, represented as $\mathcal{G}_n\in\mathbb{R}^{200\times200\times160}$, $n = 1,2,3$, respectively. Each of these tensor is stacked by 160 slices of gradient maps; (d) The unfolding matrices of the gradient maps, represented as $\G_n$, $n=1,2,3$; (e) The singular values of the matrices in (d) and the frequency of their elements; (f) A illustration of subspace basis matrices $\G_n\V_n$s, $n = 1,2,3$, implicitly calculated by the proposed E-3DTV regularizer, and the frequency of their elements.}
\label{fig_structure_sparsity}
%\vspace{-2mm}
\end{figure*}

%% 这部分要不要再凝练一��?Exploring and encoding prior structures underlying an HSI constitutes one of the key components for constructing an effective method for HSI processing and restoration tasks, including HSI denoising and compressed sensing. Many useful prior has thus been explored  \cite{Yuan2012Hyperspectral,Othman2006Noise,Zhong2013Multiple}. One of the most widely utilized and effective one for this prior setting task is the local smoothness along both HSI spatial and spectral modes.
%and the global correlation along spectrum (GCS).

One of the most widely utilized priors for HSI denoising/compressed sensing is the local smoothness prior.
The spatial local smoothness prior refers to the fact that similar objects/scenes (with shape) are often adjacently distributed with similar spectrum wave, and the spectral smoothness prior delivers the fact that adjacent bands of images of an HSI are usually collected with similar sensor parameter settings, and thus with similar values.
Such local smoothness prior structure possessed by an HSI can be equivalently understood as the sparsity of the gradient maps calculated along both the spatial and spectral modes of the HSI, as shown in Fig. \ref{fig_structure_sparsity}. The local smoothness prior of HSI space and spectrum can then be naturally encoded as the total variances along different modes of the HSI ${\X}\in\mathbb{R}^{hw\times s}$ ($h,w,s$ are the sizes of the spatial height, width and spectrum of the HSI, respectively), that is, the $\ell_1$ norm on the gradient maps $\G_n\in\mathbb{R}^{hw\times s}$ ($n=1,2,3$), where $\G_1$, $\G_2$ and $\G_3$ represent the gradient maps of $\X$ calculated along  its spatial height, width and spectrum modes, respectively\footnote{The definitions for $\G_n, n = 1,2,3$ will be provided in detail in Section 3.}. This term is called 3D-total variation regularization, or 3DTV briefly, with the form:
\begin{equation}\label{3DTV}
  \|{\X}\|_{\scriptsize{\mbox{3DTV}}}=\sum_{n=1}^3\|\G_n\|_1.
\end{equation}
It has been extensively shown that this term can be very helpful for various HSI processing issues \cite{He2016,Wang2017Hyperspectral,Hes2016}. Some even attain state-of-the-art performance for certain tasks.

%%(建议画一个从原图��?方向梯度图的计算图） 这里最 好把X和Gi 的定义和大� ��都说明一下，后面读者就有直观印象了 我们可以把gradient map称之为gradient map in spectrum, gradient map in spatial width, 和gradient map in spatial height

Although being successfully utilized in various tasks, the 3DTV regularization still has not sufficiently considered the insightful sparsity structure knowledge underlying the gradient maps of an HSI. The most typical limitation reflects on its similar sparsity consideration (i.e., identical and independently distributed Laplacian prior distribution set on gradient maps) imposed on all bands of the gradient maps in spectrum, spatial width and height. That is, the sparsity extents on all bands of these gradient maps are implicitly assumed to be similar and unrelated. This, however, is always deviated from the real cases. On one hand, in different bands of the gradient maps, instead of identical, the sparsity exist evident variations, which can be easily observed from Fig. \ref{fig_structure_sparsity}(c). On the other hand, instead of independent, there are evident correlations across different bands of the gradient maps of an HSI, as clearly shown in Fig. \ref{fig_structure_sparsity}(e), which are naturally succeeded from the band correlation property of the original HSI \cite{ZHY2014HSI,XY2016Hyperspectral}. Such a prior knowledge deviation from the real HSIs makes the capability of such a useful regularization term still has a large room to be further strengthened.

To alleviate this issue, this paper proposes an enhanced 3DTV regularization term. While similar to the conventional 3DTV, the proposed term is also operated on gradient maps of height, width and spectrum mode of HSI, the new term applies the sparsity measure ($\ell_1$ norm, just like the conventional 3DTV) on their subspace basis maps along bands instead of the gradient maps themselves.
%While similar to the conventional 3DTV, the proposed term is also calculated on the gradient maps in spectrum, spatial width and height mode of an HSI, the new term applies the sparsity measure ($\ell_1$ norm, just like general 3DTV) on, instead of the gradient maps themselves, their subspace basis maps along bands.
Each basis is obtained by a linear combination of the original gradient map vectors along bands: $\G_n\V_n$, where $\V_n$ is the transformation matrix, with size $s\times r$, facilitating to find the $r$ ($< s$) bases of $\G_n$. Such new regularization term rationally reflects practical correlated and different sparsity extents across different bands of the gradient maps.
Specifically, representing the gradients by such calculated less bases delivers the sparsity correlation insight of an HSI, and their different coefficients $\V_n$ (which will be introduced in detail in Section 4) represents the difference among bands of these gradient maps. Besides, just like other subspace learning methods, the gradient bases always contain much less bands as compared with original gradient maps, making them less negatively influenced by embedded corruptions, and more robust for different HSI processing tasks in practical scenarios.

In summary, this paper makes mainly three-fold contributions as follows:

\begin{enumerate}
\item[1)] We propose an enhanced 3DTV (E-3DTV, briefly) regularization term, to better reflect the sparsity insight of the gradient maps underlying an HSI. The non-i.i.d. sparse distribution prior characteristic of gradient maps along all HSI bands is more faithfully represented by the new regularization, and thus it is expected to  replace the 3DTV term of a general HSI processing model to improve its performance.

\item[2)] By easily using the proposed E-3DTV as the unique regularization term, we formulate two models for two typical HSI processing tasks: HSI denoising and compressed sensing. The ADMM algorithms are readily designed to solve the corresponding models. The closed-form updating equations are deduced for each step of both algorithms, and they can thus be implemented efficiently.

\item[3)] Comprehensive experimental results substantiate the superiority of the proposed method beyond state-of-the-art methods previously designed on both tasks. It is verified that such easy substitution of E-3DTV to 3DTV can always bring significant performance improvements on most of experiments.
\end{enumerate}

The remainder of this paper is organized as follows. Section 2 reviews related work. Section 3 introduces some notations and preliminaries related to this work. Section 4 presents the new regularization term, as well as the rationality of its formulation. Section 5 gives the HSI denoising model by involving E-3DTV regularization term, and Section 6 demonstrates the corresponding experimental results. Section 7 provides the model for HSI compressed sensing using the E-3DTV as its regularization, and Section 8 reports the related experimental results. Conclusions are finally made in Section 9. Throughout the paper, we denote scalars, vectors, matrices and tensors as the non-bold, bold lower case, bold upper case and upper cursive letters, respectively.

%%============================================================================
\section{Related Work}
\label{sec:1}

In the following, we will review the most relevant studies on two investigated HSI processing tasks, HSI denoising and HSI compressed sensing, respectively.

\subsection{HSI Denoising}

A number of methods  such as K-SVD \cite{Elad2006Image}, non-local means \cite{Buades2005A}, BM3D \cite{Dabov2007Image} and wavelet shrinkage \cite{Othman2006Noise} have been developed for 2D image denoising. These methods can be directly applied to HSI denoising by denoising each band of image independently. However, such an easy manner ignores the correlations among the spectral bands or spatial pixels, and thus usually could not provide satisfactory results.

More substantial progress on HSI denoising have been made by directly constructing models on HSIs, especially considering their spectral information. Most of them considered to construct a rational prior to deliver intrinsic structures underlying HSI. One representative prior is global spectral correlation. This prior indicates that the images located at different HSI bands are generally closely correlated. A generally adopted trick to formulate such a prior is first to unfold the HSI along its spectral mode to form a HSI matrix, and then use low-rank regularization terms to encode this prior \cite{Chen2011Denoising}, \cite{Gu2014Weighted}, \cite{He2015Hyperspectral}, \cite{Peng2014Reweighted}, \cite{XY2016Hyperspectral}, \cite{ZHY2014HSI}. Such a low-rank structure can be formulated implicitly into an HSI denoising model by a nuclear norm (e.g., the NNM method \cite{Candes2009Robust}, \cite{Wright2009Robust} ) or its non-convex variations, like the weighted nuclear norm (e.g., the WNNM method \cite{Gu2014Weighted}, \cite{Peng2014Reweighted}) and the Schatten $p$-norm (e.g., the WSNM method \cite{XY2016Hyperspectral} ). Besides, it can also be directly encoded as an explicit low rank matrix factorization form, like that used in the LRMR method \cite{ZHY2014HSI,Zhou2011GoDec}. Very recently, to more faithfully deliver the multi-factor affiliation underlying HSI, some HSI denoising methods \cite{Lu2017Tensor}, \cite{Wang2017Hyperspectral} treated an HSI directly as its original 3-mode tensor form, and use certain low-rank tensor approximation techniques \cite{Kolda2009Tensor}, e.g., the Tucker and CP decompositions, to characterize such low-rank characteristic under tensor expressions.

The spatial nonlocal similarity prior is another widely utilized prior for the task. A practical HSI always contains a collection of similar local patches all over the space, composing of homologous aggregation of micro-structures. Such a non-local spatial similarity prior among HSI patches is also commonly used in the HSI denoising model. Specifically, by extracting the common patterns among these nonlocally similar patches, the spatial noise is expected to be prominently alleviated. There are mainly two manners to encode such prior knowledge. One is by dictionary learning, which aims to build each local patch similar group of an HSI by the linear combination of a small number of atoms from a dictionary \cite{Mairal2009}, \cite{Liu2013A}, \cite{Qian2013Hyperspectral}, \cite{Xing2011Dictionary}. It has been empirically verified that such modeling can lead to a good performance of HSI denoising. Typical methods designed in this way include BM4D \cite{Maggioni2013Nonlocal} and TDL \cite{Peng2014Decomposable}. Another rational way to encode such non-local similarity prior is to specify a low-rank regularization term on these similar groups, and use low-rank matrix/tensor approximation techniques to realize HSI recovery \cite{Dong2015Denoising}, \cite{Mairal2009},  \cite{xie2016multispectral}.

Besides, spatial local-smoothness prior is also a powerful prior utilized for HSI denoising. The gray values of an HSI collected from natural scenes are generally continuous along its spatial and spectral modes. Such prior structure can be readily characterized by the well known total variation (TV) regularization term \cite{Rudin1992Nonlinear}, \cite{Chambolle2004An}. Multiple HSI denoising methods \cite{Bouali2011Toward}, \cite{He2016}, \cite{Wu2017Structure}, \cite{Wang2017Hyperspectral}, \cite{Jiang2016Hyperspectral} have been raised by utilizing such a regularization, and the utilization of such prior has been verified to be helpful for the fine denoising performance. Besides, due to the convexity and conciseness of such a regularization term, it is easy to integrate this term with other prior terms into a unified model to enhance the capability of the model on HSI recovery. A typical example is the LRTV method \cite{He2016}, which combines this smoothness prior with the global correlation prior along spectrum, and obtained superior performance compared with other TV-based HSI denoising approaches.

\subsection{HSI Compressed Sensing}

HSI compressed sensing aims to possibly precisely reconstruct an HSI from a small set of compressed measurements imposed on it \cite{Candes2006Robust,Donoho2006Compressed}, which is a typical inverse problem, and needs to make it calculable by setting regularization/prior terms on the to-be-recovered HSI variables.
%A HSI generally is highly compressible since both its spatial image and spectral cubes exist large redundancy \cite{Wang2017Compressive}.

Based on the prior knowledge that an HSI can be sparsely represented under an appropriate redundant dictionary, Duarte and Baraniuk \cite{Duarte2010Kronecker} proposed a compressed sensing method based on Kronecker product. They considered to use Kronecker products from sparse basis and sensing matrix. For a multi-dimensional signal, its $d$-section is defined to be the part where all dimension indicators other than the $d$-th dimension are fixed, than a sparse base of a multi-dimensional signal may consist of Kronecker products of sparse bases for each $d$-section.

%\textcolor{red}{For a multidimensional signal, a simple sparse base for it can be composed of Kronecker products of sparse bases for each of its d-parts. (hard to understand, need to be rewrote)}

Similar to the HSI denoising method, the global spectral correlation and spatial smoothness priors have also been considered for the task. In 2011, Waters et al.\cite{Waters2011SpaRCS} proposed a low rank and sparse matrix recovery method under compressed sensing, and applied the method to video and HSI compression recovery. Golbabaee et al. \cite{Golbabaee2012Hyperspectral} further discussed the correlation of HSIs in the spectral dimension, and shows certain joint sparseness in the spatial dimension under a wavelet representation and proposed a method based on low rank and joint sparse matrix recovery. In 2013, Golbabaee et al. \cite{Golbabaee2013Joint} proposed the joint norm and TV norm minimization model to encode the correlation of the spectral bands and local smoothness prior. Gogna et al. \cite{Gogna2015Split} designed the split-Breggman algorithm with low rank and sparse or joint sparse signal recovery, combined with the Kronecker product method, and applied it to the compression recovery of HSI images.

Besides, spectral correlation priors have also been used for the task by directly setting the HSI as a tensor and constructing tensor processing models  \cite{Yang2015Compressive}, \cite{Wang2017Compressive} for the problem in the recent years. For example, Yang et al. \cite{Yang2015Compressive}  used a  tensor sparse representation of a HSI cube,  under nonlinear compressed operator, which got some comparable results. To further ameliorate the performance especially under lower sampling rates, Wang et al. proposed the JTenDTV  \cite{Wang2017Compressive} model by jointing the Tucker decomposition and  3DTV norm to character the correlation of spectral bands as well as the spatial local smoothness prior, and got the comparable result.

Additional, there are also some methods proposed based on combining the compressed sensing with other HSI tasks to improve the performance of compressed sensing. The common way of fusion is to use unmixing \cite{Dias2012A}, \cite{Li2012A} to improve the performance of HSI compressed sensing.

%\subsection{\textcolor{red}{Sparsity measures for local smoothness ??}}
%
%\textcolor{red}{TV, 3DTV, group TV, LRTV}
%% ===========================================================================

%%============================================================================
\section{Notations and Preliminaries}

For a given HSI $\mathcal{X}\in\mathbb{R}^{h\times w\times s}$, where $h$, $w$ and $s$ denote the sizes of spatial height, width and  number of spectrum in the HSI, respectively. We denote the unfolding matrix of $\mathcal{X}$ along the spectral mode as $\X\in\mathbb{R}^{hw\times s}$, which satisfies $\X = \mbox{unfold}_3(\mathcal{X})$ and $\mathcal{X} = \mbox{fold}_3(\X)$.

 Let $\mathcal{X}(i,j,k)$ denote the intensity at the voxel $(i,j,k)$, and let
 \begin{equation}\label{Diff}
 \begin{split}
     \mathcal{G}_1(i,j,k) =  & \mathcal{X}(i,j,k)-\mathcal{X}(i+1,j,k) \\
     \mathcal{G}_2(i,j,k) =  & \mathcal{X}(i,j,k)-\mathcal{X}(i,j+1,k) \\
     \mathcal{G}_3(i,j,k) = & \mathcal{X}(i,j,k)-\mathcal{X}(i,j,k+1)
 \end{split}
 \end{equation}
 denote three difference operations at the voxel $(i,j,k)$ along the spatial height, width and the spectrum, respectively.
We can now introduce the difference operations calculated three different modes on an HSI $\mathcal{X}$, i.e., $D_n(\cdot)$, as follows:
\begin{equation}\label{TensorG}
  \mathcal{G}_n = D_n(\mathcal{X}_n), \forall n = 1,2,3,
\end{equation}
where $\mathcal{G}_n\in \mathbb{R}^{h\times w\times s}$, $n = 1,2,3$ are the gradient-map tensors  corresponds to spatial height, weight and spectrum modes respectively\footnote{Here, we use zero padding for $\mathcal{X}$ before applying difference operation on it, which can keep the size of $\mathcal{G}_n$ the same as $\mathcal{X}$ and make calculation convenient in the later sections.}.
Then, we denote the unfolding matrices of these gradient maps as:
\begin{equation}\label{MatrixG}
  \G_n  = \mbox{unfold}_3(\mathcal{G}_n), \forall n = 1,2,3.
\end{equation}
One can also see Fig. \ref{fig_structure_sparsity} (b) (c) and (d) for easy understanding these operations.

\begin{figure*}
%\vspace{-1.5mm}
\centering
\includegraphics[width=1\linewidth]{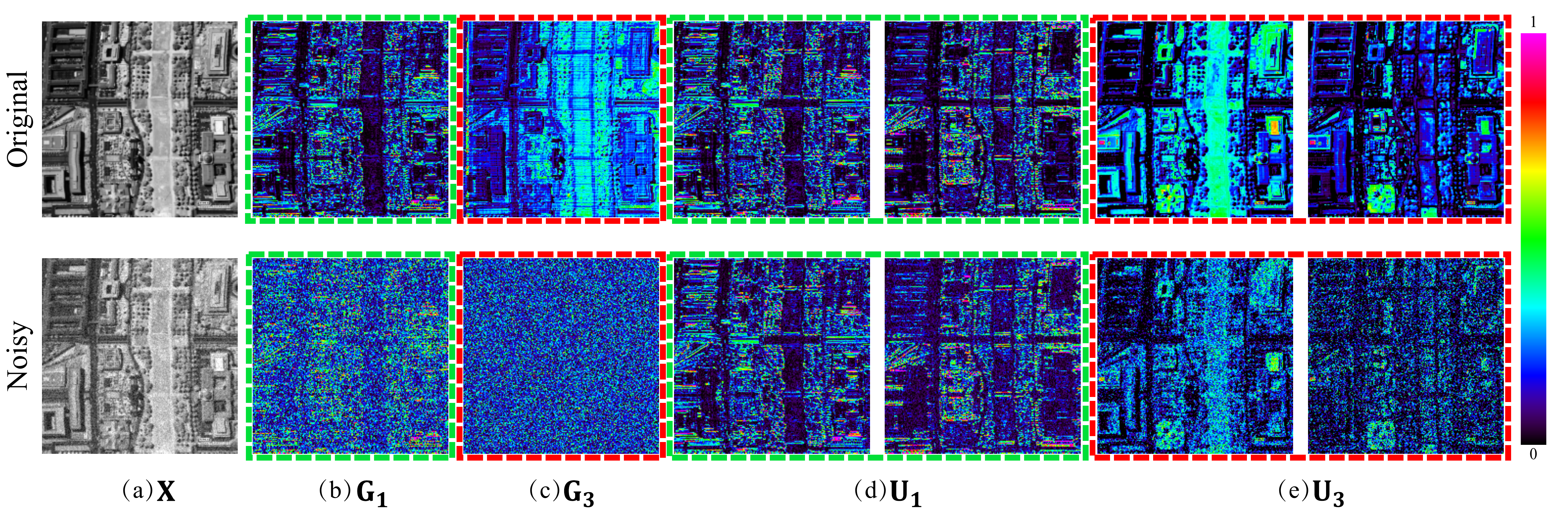}
\vspace{-6.5mm}
\caption{(a) A band of a clean HSI (upper), and the same band with Gaussian noise added to it (lower); (b)-(c) The corresponding gradient maps of images in (a), calculated along spatial height ($\nabla_1\X$) and spectral ($\nabla_3\X$) modes, respectively. (d)-(e) Illustration of two typical basis gradient maps $\U_1$ and $\U_3$ calculated on the original and noisy HSI. Since each column in $\U_n$ is of size $hw\times 1$ which can be folded into $h\times w$ for easy observation.}
%{From left to right: the height-mode gradient map of the HSI unfolded along its spectral mode, the basis matrix containing a set of sparse basis vectors, and the coefficient matrix. In the upper right, two basis vectors of U are folded to 2D images, which represent fundamental structures underlying the gradient map $\nabla{\mathbf{X}}$. }
\label{fig_lowrank}
%\vspace{-2.5mm}
\end{figure*}
%\vspace{-1.5mm}

It is easy to find that the difference operations $D_1$ and $D_2$ on $\mathcal{X}$ is equivalent to applying subtractions between rows in $\X$, and $D_3$ on $\mathcal{X}$ is equivalent to applying subtractions between columns in $\X$. This means all the three different operations are linear\footnote{Subtractions between rows of a matrix is equivalent to pre-multiply the matrix by a proper matrix, and subtractions between columns of a matrix is equivalent to post-multiply the matrix by a proper matrix.}. We can then denote following three linear operations:
\begin{equation}\label{D}
  \nabla_n\X =\mbox{unfold}_3\left(D_n\left(\mbox{fold}(\X)\right)\right)=\G_n, \forall n = 1,2,3.
\end{equation}
Note that this linear operations encode the relationship between an HSI $\X\in\mathbb{R}^{hw\times s}$ and its gradient maps $\G_n\in\mathbb{R}^{hw\times s}$s, facilitating us to construct HSI processing method with priors on the gradient maps.

The most commonly used sparsity measure on the gradients maps, $\G_n$s, is the so call $\ell_1$ norm sparse measure:
\begin{equation}\label{L1}
  S\left(\G_n\right) = \|\G_n\|_1.
\end{equation}
By adopting this sparsity measure, the widely used smoothness regularization term for HSI, 3DTV \cite{Wang2017Hyperspectral}, \cite{Wang2017Compressive}, can be constructed by
\begin{equation}\label{3DTV}
  \|\X\|_{\scriptsize{\mbox{3DTV}}}=\sum_{n=1}^3S(\nabla_n\X)=\sum_{n=1}^3S(\G_n).
\end{equation}
One can also see Fig. \ref{fig_structure_sparsity} for easy understanding this term.
We can then present the enhanced 3DTV term beyond this conventional one in the next section.

%One of the most common used smoothness-regularization term for HSI is the 3DTV norm \cite{}, defined as
%\begin{equation}\label{3DTV}
%  \|\mathcal{X}\|_{\scriptsize{\mbox{3DTV}}}=\sum_{n=1}^3\|\G_n\|_1.
%\end{equation}
%This term can also be seen as a sparse regularization for the gradient maps $\G_n$s, with the
%so call $l_1$ norm sparse measure:
%\begin{equation}\label{L1}
%  S\left(\G_n\right) = \|\G_n\|_1.
%\end{equation}
%Although easy to compute and be applied to different task, this simple sparse regularization term fails to reflect the relationship among different column in $\G_n$s, especially in HSI restoration and processing tasks.

\section{Enhanced 3DTV Regularization
}

We first discuss how to construct an enhanced sparsity term for each $\G_n$, which will help us construct the E-3DTV norm in the following sections.
%\subsection{Enhanced Sparsity Regularization Term}
\subsection{Enhanced TV Regularization Term }

For an given HSI whose unfolding matrix is $\X\in\mathbb{R}^{hw\times s}$, we seek to model the correlated sparse structure along its gradient map, i.e., our goal is to construct a rational sparsity measure $S_{\footnotesize{\mbox{E-TV}}}(\G_n)$ for each gradient map $\G_n = \nabla_n \X$, $n=1,2,3$.

As studied in many previous works, the unfolded HSI $\X$ can be rationally assumed to be of low rank \cite{Gu2014Weighted}, \cite{He2015Hyperspectral}, \cite{Peng2014Reweighted}, \cite{XY2016Hyperspectral}, \cite{ZHY2014HSI}, \cite{He2016}. Then, it is easily to find that $\G_n$  also possesses  low-rankness property\footnote{ Since $\nabla_n(\cdot)$ is a linear operation, we have $\mbox{rank}(\nabla_n\X)\leq\mbox{rank}(\X)$.
%This is because of  $\mbox{rank}(\A\B) \leq \min\left\{\mbox{rank}(\A), \mbox{rank}(\B)\right\}$ for any two matrices $\A$ and $\B$.
}. This means $\G_n$ can be expressed in following low-rank matrix factorization form:
\begin{equation}\label{UV}
  \G_n=\U_n\V_n^T,
\end{equation}
where $\U_n\in\mathbb{R}^{hw\times r}$ and $\V_n\in\mathbb{R}^{s\times r}$, $r$($<hw,s$) is the rank of the HSI. In addition, it easy to find the columns in $\U_n$ is in fact a set of basis to the gradient maps $\G_n$.
%A most common way to calculate this matrix factorization is to exploit singular value decomposition (SVD): $\G_n = \tilde{\U}_n\mathbf{\Sigma}_n \tilde{\V}_n^T $, then set $\U_n$ as the firsts $r$ columns of $\tilde{\U}_n\mathbf{\Sigma}$ and set $\V_n$ as the firsts $r$ columns of $\tilde{\V}_n$. In this case,
We can also rationally assume $\V_n^T\V_n=I$ to make the subspace bases contain possibly compensate while non-repetitive information underlying the original $\G_n$.
In this case, we can rewrite equation (\ref{UV}) in following formulation:
\begin{equation}\label{GV}
  \U_n=\G_n\V_n,
\end{equation}
where $\U_n$ is now represented as a linear combination of all columns of gradient maps. Furthermore, to guarantee the bases keep sufficient information of the original $\G_n$, we further assume $\|\U_n\|_F$, i.e., $\|\G_n\V_n\|_F$, is of approximately similar capacity with the original $\|\G_n\|_F$.

Fig. \ref{fig_lowrank} shows the bases calculated on a clean HSI and it noisy observation. It is easy to observe from Fig. \ref{fig_lowrank} (b) and (c) that by adding a little simulated Gaussion noise to the HSI, the quality of gradient maps, as well as its sparsity structures, can be badly damaged, which makes the corrupted images not easily to be restored by only using the original 3DTV regularization term (\ref{3DTV}) directly imposed on it. However, we can easily observe form Fig. \ref{fig_lowrank} (d) and (e) that $\U_n$ is evidently less corrupted than original gradient maps, and its sparse structure is evidently more evident and less hampered by noises. It is thus expected to be easier to recover the bases $\U_n$ from the corrupted ones by E-3DTV than to do this task on the original gradient maps by conventional 3DTV.

%. Moreover, we can observe that the columns in $\U_n$ of the original HSI are sparse, and it is much easier to restore the original $\U_n$ by adopting sparse regularization on it, than to directly restore $\G_n$ itself with spare regularization.

Thus, instead to design sparse regularization term for $\G_n$ itself, we seek to design sparse regularization for $\U_n$, i.e. for $\G_n\V_n$. Based on the aforementioned analysis, the E-3DTV measure for each $\G_n$, $n = 1,2,3$, is of the following implicit expression formulation:

%Since the low-rank matrix factorization can have infinite feasible solutions \footnote{ For any $\U$ and $\V$ satisfy $\A = \U\V^T$, $\hat{\U} = \U\B^{-1}$ and $\hat{\V}\B$ also is a solution for the low-rank factorization of $\A$, where $\B\in \mathbb{R}^{r\times r}$ is any invertible matrix. The result obtained by SVD is only one of the feasible solutions.}, we then construct our E-3DTV measure for $\G_n$, $n = 1,2,3$, in following  implicit expression formulation:

\begin{equation}
\begin{split}\label{CSC_new}
S_{\scriptsize{\mbox{E-TV}}}(\G_n) &= \min_{\V_n \in \mathbb{R}^{s\times r}} \| \G_n\V_n \|_1 \\
&{ \rm s.t.}~ \|\G_n\V_n\|_F=\|\G_n\|_F, \mathbf{V}_n^T\mathbf{V}_n=\mathbf{I}.
\end{split}
\end{equation}
Here, $\V_n$ can be seen as a coefficient matrix that relates different gradient maps together.
%The constant $\|\U_n\|_F =\|\G_n\V_n\|_F= \|\G_n\|_F$ is assumed to guarantee the bases keep sufficient information of the original $\G_n$.

%The orthogonal constraint and constraint $\|\G_n\V_n\|_F= \|\G_n\|_F$ are adopted to alleviate the scaler freedom of the solution and prevent $\G\V_n$ become zero. These constraints can also ensure $\G_n =\G_n\V_n\V_n^T= \U_n\V_n^T$.It should be noted that the $\V_n$ here can be different as the result calculated with SVD, which can lead to more sparse basis $\U_n = \nabla_n\X\V_n$ for the gradient maps.

It should be noted that the proposed E-3DTV measure is actually a linear transformed sparsity on $\G_n$, with the transforming matrix $\V_n$ to be determined by the input data automatically.
Compared with the traditional $\ell_1$ sparsity measure (\ref{L1}), the proposed  (\ref{CSC_new}) takes the correlation prior among gradient maps into consideration when constructing sparse regularization term. It takes the advance that there exist some low-rank basis for $\G_n$s, which is sparse but can represent all information in $\G_n$ with a coefficient matrix $\V_n$, and these bases are always more stable to be calculated than the gradient maps themselves when the HSI is corrupted.
%It also should be noted that this measure can be only adopt to low-rank data, since it is not possible to satisfy $\|\G_n\V_n\|_F= \|\G_n\|_F$ when the rank of $G_n$ is larger than $r$.
%这句��?最好还是做实验��?把问题的解解一��?%Compare to (\ref{W3DTV}), the proposed (\ref{CSC_new}) doesn't have parameters that need to be determined before hand. The coefficient matrix $\V_n$ can be determined by the input data automatically. Moreover, our spare regularization term is adopt on the basis, which no longer troubles by the sparse-degree difference among different gradient maps.

Since it is not easy to  solve (\ref{CSC_new}) directly, we reformulate (\ref{CSC_new}) as the following equivalent formulation:
\begin{equation}
\begin{split}\label{CSC}
S_{\scriptsize{\mbox{E-TV}}}(\G_n) &= \min_{\U_n,\V_n}\|\mathbf{U}_n\|_1 \\
&{\rm s.t.}~ \G_n=\mathbf{U}_n\mathbf{V}_n^T, \mathbf{V}_n^T\mathbf{V}_n=\mathbf{I},\\
&\mathbf{U}_n \in\mathbb{R}^{hw\times r}, \mathbf{V}_n \in\mathbb{R}^{s\times r}.\\
\end{split}
\end{equation}
Here, we call two problems equivalent if from a solution of one, a solution
of the other is readily found, and vice versa.
The proof of the equivalence here is presented in theorem \ref{theorem_e} in the supplementary material.

\subsection{Enhanced 3DTV Regularization Term}

By adopting the proposed sparsity measure (\ref{CSC}), we can now construct E-3DTV regularization term imposed on gradient maps calculated along all its spectral, spatial width and height modes. In similar way as 3DTV regularization (\ref{3DTV}), we model our regularization term by summing the sparsity measure of the gradient maps along different modes:
\begin{equation}\label{Reg}
  \|\X\|_{\scriptsize{\mbox{E-3DTV}}} =\sum_{n=1}^3S_{\scriptsize{\mbox{E-TV}}}(\nabla_n\X)=\sum_{n=1}^3S_{\scriptsize{\mbox{E-TV}}}(\G_n).
\end{equation}
One can also see Fig. \ref{fig_structure_sparsity} for easy understanding this term.
Moreover, this term is equivalent to:
\begin{equation}\label{Reg2}
  \begin{split}
  \|\X\|_{\scriptsize{\mbox{E-3DTV}}}   &=  \sum_{n=1}^3\min_{\U_n,\V_n}\|\mathbf{U}_n\|_1\\
%  =\min_{\U_n,\V_n} \sum_{n=1}^3\|\mathbf{U}_n\|_1 \\
&{\rm s.t.} \nabla_n\X=\mathbf{U}_n\mathbf{V}_n^T, \mathbf{V}_n^T\mathbf{V}_n=\mathbf{I},\\
&  \mathbf{U}_n \in\mathbb{R}^{hw\times r}, \mathbf{V}_n \in\mathbb{R}^{s\times r}, n = 1,2,3.\\
\end{split}
\end{equation}

%In later sections, we will show that we can easily exploit the discussed prior structure in different HSI processing and restoration tasks by adopting regularization term (\ref{Reg2}).
In the following sections, we will verify that the proposed E-3DTV regularization term (\ref{Reg2}) can be easily embedded into HSI processing models and the corresponding resolving algorithms can be easily designed.

\section{E-3DTV methods for HSI denoising}
\subsection{The E-3DTV based HSI denoising  model}

The denoising task is to separate the clean data and noise from the noisy data. In many of the denoising tasks, the noise is often assumed to be Gaussian, and a quadratic loss function between the clean data and noisy data is thus often exploited. However, in real hyperspectral scenes, the types of noise are various, such as sparse noise, stripes noise, dead lines, missing pixels, and so on \cite{XY2016Hyperspectral,He2016}, which obviously do not obey a Gaussian distribution. Therefore,  we choose the  $\ell_1$ norm, a more robust loss function to general heavy noises \cite{Wright2009Robust,Candes2009Robust}, to model the noise and construct our denoising model as:
\begin{equation}
\begin{split}\label{denoise_model0}
&\min_{\mathbf{X},\mathbf{E}} \tau\|\X\|_{\scriptsize{\mbox{E-3DTV}}}+\|\mathbf{E}\|_1 \\
&~{\rm s.t.}~\mathbf{Y}=\mathbf{X}+\mathbf{E},\\
\end{split}
\end{equation}
where $\mathbf{Y}$, $\X$, $\mathbf{E}\in \mathbb{R}^{hw \times s}$ denote the input noisy HSI,  to-be-reconstructed HSI and the noise embedded on it, respectively. $\mathbf{\tau}$ is tradeoff parameter between the proposed E-3DTV measure and the noise level. Based on (\ref{Reg2}), it is easy to find that this model is equivalent to:
\begin{equation}
\begin{split}\label{denoise_model}
&\min_{\mathbf{X},\mathbf{U}_n,\mathbf{V}_n,\mathbf{E}} \tau\sum_{n=1}^3\|\mathbf{U}_n\|_1+\|\mathbf{E}\|_1 \\
&~{\rm s.t.}~\mathbf{Y}=\mathbf{X}+\mathbf{E}, \\
&\qquad \nabla_n\mathbf{X}=\mathbf{U}_n\mathbf{V}_n^T, \mathbf{V}_n^T\mathbf{V}_n=\mathbf{I}\
\\
&\qquad \mathbf{V}_n \in\mathbb{R}^{s\times r}, \mathbf{U}_n \in\mathbb{R}^{hw\times r}, n=1,2,3. \\
\end{split}
\end{equation}
This optimization problem (\ref{denoise_model}) can be readily solved by the ADMM strategy as follows.

\subsection{The ADMM Algorithm }
To slove the proposed HSI denoising model by ADMM algorithm, the following augmented Lagrangian function is  firstly required to be minimized:
\begin{equation}
\small
\begin{split}\label{denosie_lag}
&\mathcal{L}(\mathbf{X},\mathbf{E},\mathbf{U}_n,\mathbf{V}_n,\mathbf{M}_n,\mathbf{\mathbf{\Gamma}}) =
\sum_{n=1}^3 \|\mathbf{U}_n\|_1 +\lambda \|\mathbf{E}\|_1\\
&+\sum_{n=1}^3{\langle \mathbf{M}_n, \nabla_n\mathbf{X}-\mathbf{U}_n\mathbf{V}_n^T\rangle + \frac{\mu}{2}\|\nabla_n\mathbf{X}-\mathbf{U}_n\mathbf{V}_n^T\|^2_F}\\
& +\langle \mathbf{\mathbf{\Gamma}},\mathbf{Y}-\mathbf{X}-\mathbf{E}\rangle+\frac{\mu}{2}\|\mathbf{Y}-\mathbf{X}-\mathbf{E}\|^2_F,
\end{split}
\end{equation}
where $\mathbf{M}_n$, $n = 1,2,3$ and $\mathbf{\mathbf{\Gamma}}$ are the Lagrange multipliers and $\mu$ is a positive scalar in ADMM algorithm.

In the ADMM framework, we need to alternatively optimize each variable involved in (\ref{denosie_lag}) with all others fixed. We then discuss how to solve its  sub-problems with respect to each involved variable.

\textbf{Update} $\mathbf{X}$. Extracting all items containing $\mathbf{X}$ from Eq. (\ref{denosie_lag}), we can obtain the sub-problem as:
\begin{equation}
\begin{split}\label{fun_x_denoise}
&\mathbf{X}^{k+1}:=\mathop{\argmin}_{\mathbf{X}} \sum_{n=1}^3\langle \mathbf{M}_n^{k}, \nabla_n\mathbf{X}-\mathbf{U}_n^{k}\mathbf{V}_n^{k^T}\rangle+ \\
&\langle \mathbf{\Gamma}^{k},\mathbf{Y}-\mathbf{X}-\mathbf{E}^{k}\rangle+\frac{\mu}{2}\sum_{n=1}^3\|\nabla_n\mathbf{X}-\mathbf{U}_n^{k}\mathbf{V}_n^{k^T}\|^2_F\\
&+\frac{\mu}{2} \|\mathbf{Y}-\mathbf{X}-\mathbf{E}^{k}\|^2_F.
\end{split}
\end{equation}

Optimizing (\ref{fun_x_denoise}) with respect to $\mathbf{X}$ can thus be equivalently treated as solving the following linear system:
\begin{equation}
\begin{split}
&(\mu\mathbf{I}+\mu\sum_{n=1}^3\mathbf{\nabla}_n^*\mathbf{\nabla}_n)(\mathbf{X}) = \mu(\mathbf{Y}-\mathbf{E}^{(k)})+\mathbf{\Gamma}^{k}\\
&\qquad +\mu\sum_{n=1}^3\mathbf{\nabla}_n^*(\mathbf{U}_n^{k}\mathbf{V}_n^{k^T})-\mathbf{\nabla}_n^*(\mathbf{M}_n^{k})\\
\end{split}
\end{equation}
where $\mathbf{\nabla}_n^*$ indicates the adjoint operator of ${\nabla}_n$. Attributed to the block-circulant
of matrix corresponding to the operator $\mathbf{\nabla}_n^*\mathbf{\nabla}_n,n=1,2,3$, it can be diagonalized by using the \textbf{FFT} matrix \cite{Ng1999A}. Then, $\mathbf{X}$ can be efficiently computed by
\begin{equation}
\small
\label{slover_x_denoise}
\left\{
\begin{split}
&\mathbf{H}_x=\mu(\mathbf{Y}-\mathbf{E}^{k})+\mathbf{\Gamma}^{k} +\sum_{n=1}^3 \!\mathbf{\nabla}_n^*(\mu\mathbf{U}_n^{k}\mathbf{V}_n^{k^T}-\mathbf{M}_n^{k}),\\
&\mathbf{T}_x=|\text{fftn}(\mathbf{D}_1)|^2+|\text{fftn}(\mathbf{D}_2)|^2+|\text{fftn}(\mathbf{D}_3)|^2, \\
&\mathbf{X}^{k+1}=\text{ifftn}\left(\frac{\text{fftn}(\mbox{Fold}(\mathbf{H}_x))}{\mu\mathbf{1}+\mu\mathbf{T}_x}\right),
\end{split}
\right.
\end{equation}
where fftn and ifftn indicate fast Fourier transform and its inverse transform, respectively.
$\left|\cdot\right|^2$is the elements-wise square, and the division is also performed element-wisely.

\textbf{Update} $\mathbf{U}_n,n=1,2,3$. Extracting all items containing $\mathbf{U}_n$ from Eq. (\ref{denosie_lag}), we can obtain:
\begin{equation}
\begin{split}\label{slover_U_denoise}
&\mathbf{U}_n^{(k+1)}:=\mathop{\argmin}_{\mathbf{U}_n} \|\mathbf{U}_n\|_1+\\
&\qquad \frac{\mu}{2}\|\mathbf{U}_n-(\nabla_n\mathbf{X}+\mathbf{M}_n^{k}/\mu)\mathbf{V}_n^{k}\|^2.\\
\end{split}
\end{equation}
This sub-problem can be efficiently solved by calling the known soft-thresholding operator \cite{Donoho2002De}:
\begin{equation}
\mathcal{S}_\mathbf{\Delta}(\mathbf{x})=\left\{
\begin{split}
{x}-\mathbf{\Delta} &   &\mbox{if} \quad\mathbf{x} > \mathbf{\Delta},  \\
{x}+\mathbf{\Delta}, &   & \mbox{if} \quad\mathbf{x} < -\mathbf{\Delta}, \\
0,                 &   & \qquad \mbox{otherwise},
\end{split}
\right.
\end{equation}
where $ x\in\mathbb{R}$ and $\mathbf{\Delta} >0 $, $\mathbf{\Delta} $ is the threshold value. Then the solution of (\ref{slover_U_denoise}) can be calculated in close-form as follows:
\begin{equation}
\label{slover_u_denoise}
\mathbf{U}_n^{k+1}=\mathcal{S}_{\frac{1}{\mu}}\left(\left(\nabla_n\mathbf{X}+\mathbf{M}_n^{k}/\mu\right)\mathbf{V}_n^{k}\right).
\end{equation}

\textbf{Update} $\mathbf{V}_n,n=1,2,3$. Extracting all items containing orthogonal $\mathbf{V}_n$ from Eq. (\ref{denosie_lag}), we have:
\begin{equation}
\begin{split}
\label{slover_V_denoise}
&\mathbf{V}_n^{k+1}:=\mathop{\argmin}_{\mathbf{V}_n}\frac{\mu}{2}\left\Vert\mathbf{U}_n^{k}\mathbf{V}_n^T-(\nabla_n\mathbf{X}+\frac{\mathbf{M}_n^{k}}{\mu})\right\Vert_F^2\\
&\qquad\quad =\mathop{\argmin}_{\mathbf{V}_n} \left\langle\left(\nabla_n\mathbf{X}+\frac{\mathbf{M}_n^{k}}{\mu}\right)^T\mathbf{U}_n^{k}, \mathbf{V}_n \right\rangle.
\end{split}
\end{equation}
The global solution of this sub-problem can be achieved in close-form by following theorem \cite{xie2016multispectral}:
\begin{theorem}\label{Theorem1}
For any $\mathbf{A}\in\mathbf{M}_{m,n}$, the global solution of the problem:
\begin{equation}
\label{theory_V}
\mathop{\min}_{\mathbf{V}\mathbf{V}^{T}=\mathbf{I}} \langle \mathbf{A},\mathbf{V} \rangle
\end{equation}
is $\mathbf{V}^*$=$\mathbf{B}\mathbf{C}^T$, where $\mathbf{A}=\mathbf{B}\mathbf{D}\mathbf{C}^T$ denotes the condensed SVD of $\mathbf{A}$.
\end{theorem}

We can then obtain the updating equation of $\mathbf{V}_n^{k+1}$ as:
\begin{equation}
\label{solver_v_denoise}
\left\{
\begin{split}
& [\mathbf{B},\mathbf{D},\mathbf{C}]= \text{svd}((\nabla_n\mathbf{X}+\mathbf{M}_n^{k}/\mu)^T\mathbf{U}_n^{k})\\
& \mathbf{V}_n^{k+1}=\mathbf{B}\mathbf{C}^T.
\end{split}
\right.
\end{equation}

\textbf{Update} $\mathbf{E}$. Extracting all items containing $\mathbf{E}$ from Eq. (\ref{denosie_lag}), we can get:
\begin{equation}
\small
\begin{split}\label{slover_E_denoise}
\mathbf{E}^{k+1}:&=\mathop{\argmin}_{\mathbf{E}} \lambda\|\mathbf{E}\|_1 + \frac{\mu}{2}\left\|\mathbf{E}-\left(\mathbf{Y}-\mathbf{X}+\frac{\mathbf{\Gamma}^{k}}{\mu}\right)\right\|^2\\
&=\mathcal{S}_{\frac{1}{\mu}}\left(\mathbf{Y}-\mathbf{X}+\frac{\mathbf{\Gamma}^{k}}{\mu}\right).
\end{split}
\end{equation}

Based on the general ADMM principle, the multipliers are further updated by the following equations:

\begin{equation}
\label{solver_m_denoise}
\small
\left\{
\begin{split}
&\mathbf{M}_n^{k+1}=\mathbf{M}_n^{k}+\mu\left(\nabla_n\mathbf{X}-\mathbf{U}_n^{k}\mathbf{V}_n^{k^T} \right),n=1,2,3\\
&\mathbf{\Gamma}^{k+1}=\mathbf{\Gamma}^{k}+\mu\left(\mathbf{Y}-\mathbf{X}-\mathbf{E}^{k}\right).
\end{split}
\right.
\end{equation}
%\begin{shrinkeq}{-2ex}
%\end{shrinkeq}

Summarizing the aforementioned descriptions, we can get the entire ADMM algorithm, as summarized in $\mathbf{Algorithm}$ $\mathbf{1}$, for solving the model (\ref{denoise_model}).
\begin{algorithm}
\caption{E-3DTV method for HSI denoising}\label{alg_denoise}
\small
\begin{algorithmic}[1]
\renewcommand{\algorithmicrequire}{\textbf{Input:}}
\renewcommand{\algorithmicensure}{\textbf{End}}
\REQUIRE The HSI $\mathbf{\mathcal{Y}}\in\mathbb{R}^{h\times w\times s}$, unfolding to the matrix $\mathbf{Y}\in\mathbb{R}^{hw\times s}$, TV unfolding matrix rank: $r$, Regularized parameters $\tau$, Stopping criteria $\epsilon_1$,$\epsilon_2$.

\renewcommand{\algorithmicrequire}{\textbf{Initialization:}}
\renewcommand{\algorithmicensure}{\textbf{End}}
\REQUIRE Initial $\mathbf{X}$, $\mathbf{E}$, $\mathbf{U}_n$, $\mathbf{V}_n$, $\mathbf{M}_{n}$, $\mathbf{M}_{4}$.

\WHILE {not converge}
\STATE
 Update $\mathbf{X}$, $\mathbf{E}$  by Eq.(\ref{slover_x_denoise}) and (\ref{slover_E_denoise}), respectively.
\STATE
Update $\mathbf{U}_{n}$, $\mathbf{V}_{n}$ by Eq.(\ref{slover_u_denoise}) and (\ref{solver_v_denoise}), respectively .
\STATE
Update the multipliers $\mathbf{M}_{n}$, $\mathbf{\Gamma}$ by Eq. (\ref{solver_m_denoise}).
\STATE
Check the convergence conditions\\
$\quad \Vert\mathbf{Y}-\mathbf{X}-\mathbf{E}\Vert_2^F/\Vert\mathbf{Y}\Vert_2^F\leq\epsilon_1$\\
$\quad \Vert\nabla_{n}\mathbf{X}-\mathbf{U}_n\mathbf{V}_n^T\Vert_2^F/\Vert\mathbf{Y}\Vert_2^F\leq\epsilon_2,n=1,2,3$\\
\ENDWHILE

\renewcommand{\algorithmicrequire}{\textbf{Output:}}
\renewcommand{\algorithmicensure}{\textbf{End}}
\REQUIRE  $\mbox{Fold}(\mathbf{X})\in\mathbb{R}^{h\times w\times s}$.
\end{algorithmic}
\end{algorithm}

%% ===========================================================================
\section{Experimental Results on HSI denoising}
In this section, extensive experiments are presented to evaluate the performance of the proposed method. We applied our proposed E-3DTV prior to HSI denoising in comparison with 8 state-of-the-art methods.
%There two different data sets were used for simulation experiments to demonstrate the universality of our method in HSI data.
In order to give an overall evaluation, three quantitative quality indices are employed: PSNR, SSIM\cite{Wang2004Image}, ERGAS\cite{Wald2010Data}. PSNR and SSIM are two conventional spatial-based metrics, while ERGAS is spectral-based evaluation measure. The larger PSNR and SSIM values are, and the smaller ERGAS values is, the better quality the restored images tends to be of.

%We selected two HSIs\footnote{Owing to HSI have more spectral bands than MSI, HSIs have stronger spectral correlation property than MSIs, so we not choose the Columbia Multispectral Image Database, the two HSIs dataset are Indian Pines and DCMall. }  to test our method under different noise levels.
To demonstrate the effectiveness of the proposed algorithm, we compared our denoising results with 8 recently developed state-of-the-art denoising methods, including non-local transform domain filter for volumetric data (BM4D) \cite{Maggioni2013Nonlocal}, TDL \footnote{\url{http://gr.xjtu.edu.cn/c/document_library/get_file?folderId=1766524&name=DLFE-38410.zip}} \cite{Peng2014Decomposable} , NNM \cite{Wright2009Robust}, WNNM \cite{Gu2014Weighted}, LRMR \cite{ZHY2014HSI}, WSNM \cite{XY2016Hyperspectral}, LRTV \cite{He2016} and LRTDTV \cite{Wang2017Hyperspectral}. All the involved parameters in these competing methods are fine-tuned by default settings or following the rules in their papers to guarantee their possibly optimal performance. Before the experiments, the gray value of data were rescaled to [0,1].  All the experiments were performed using MATLAB (R2014a) on Windows 7 OS with dual-core Inter 3.60-GHz processor and 64 GB of RAM.

\begin{figure*}[!]
\centering
\includegraphics[scale=0.55]{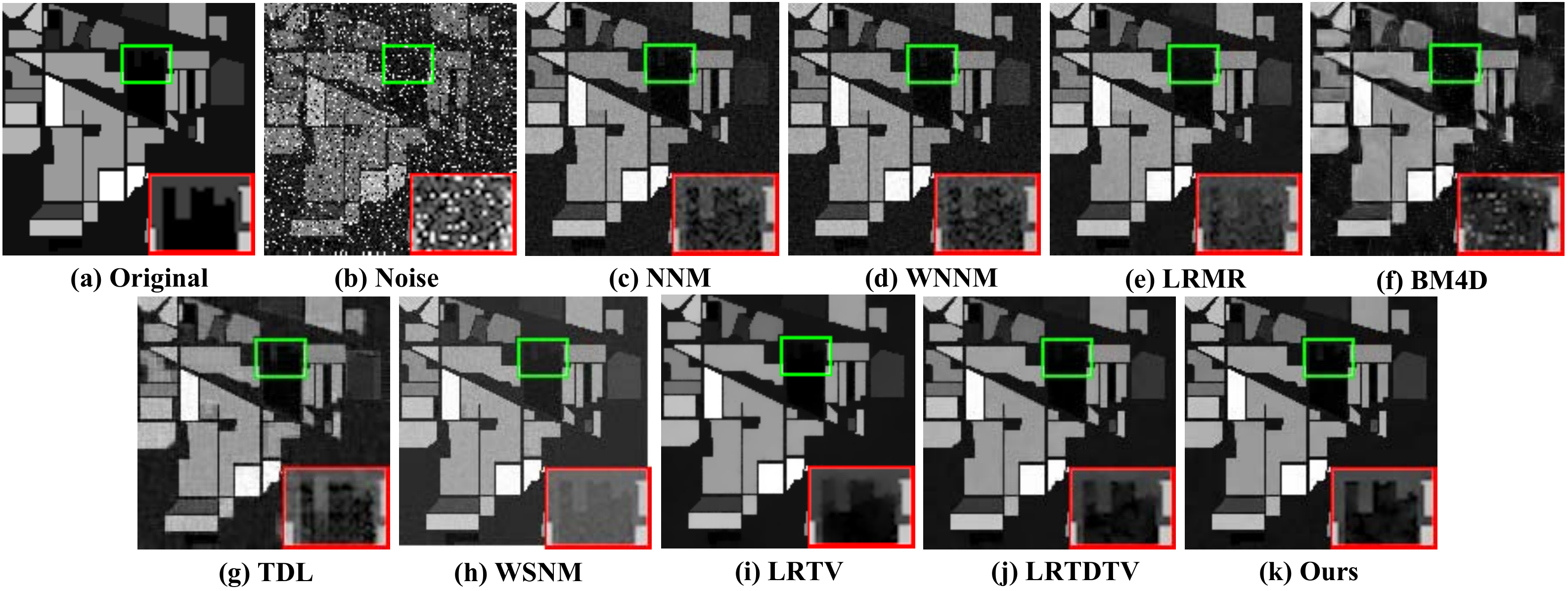}
\caption{Restoration performance of all competing methods on the band 220 of the ARIVIS Indian Pines HSI in Case c. The demarcated area is enlarged in the right bottom corner for better visualization. The figure is better seen by zooming on a computer screen.}
\label{fig_indian220}\vspace{1mm}
\end{figure*}

\begin{figure*}[!]
\centering
\includegraphics[width=1\linewidth]{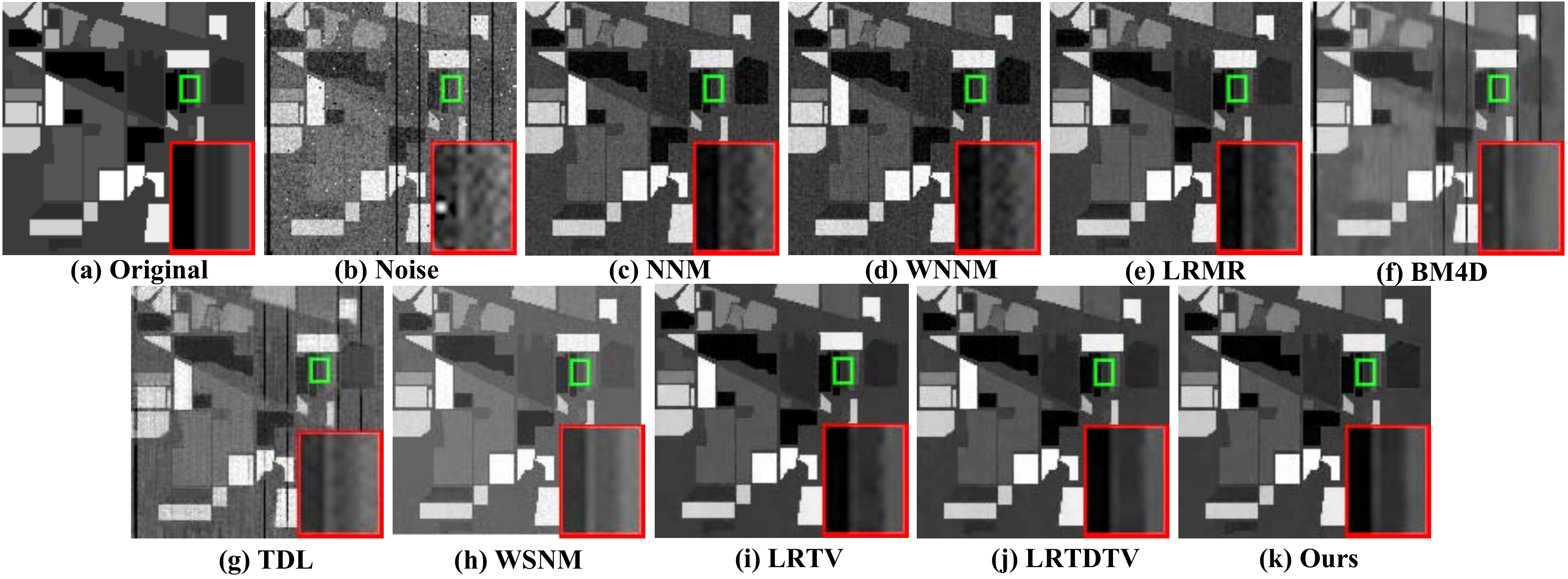}
\caption{Restoration performance of all competing methods on the band 120 of the ARIVIS Indian Pines HSI in case e.}
\label{fig_indian120}\vspace{-0mm}
\end{figure*}

\begin{figure*}[!]
\centering
\includegraphics[width=1\linewidth]{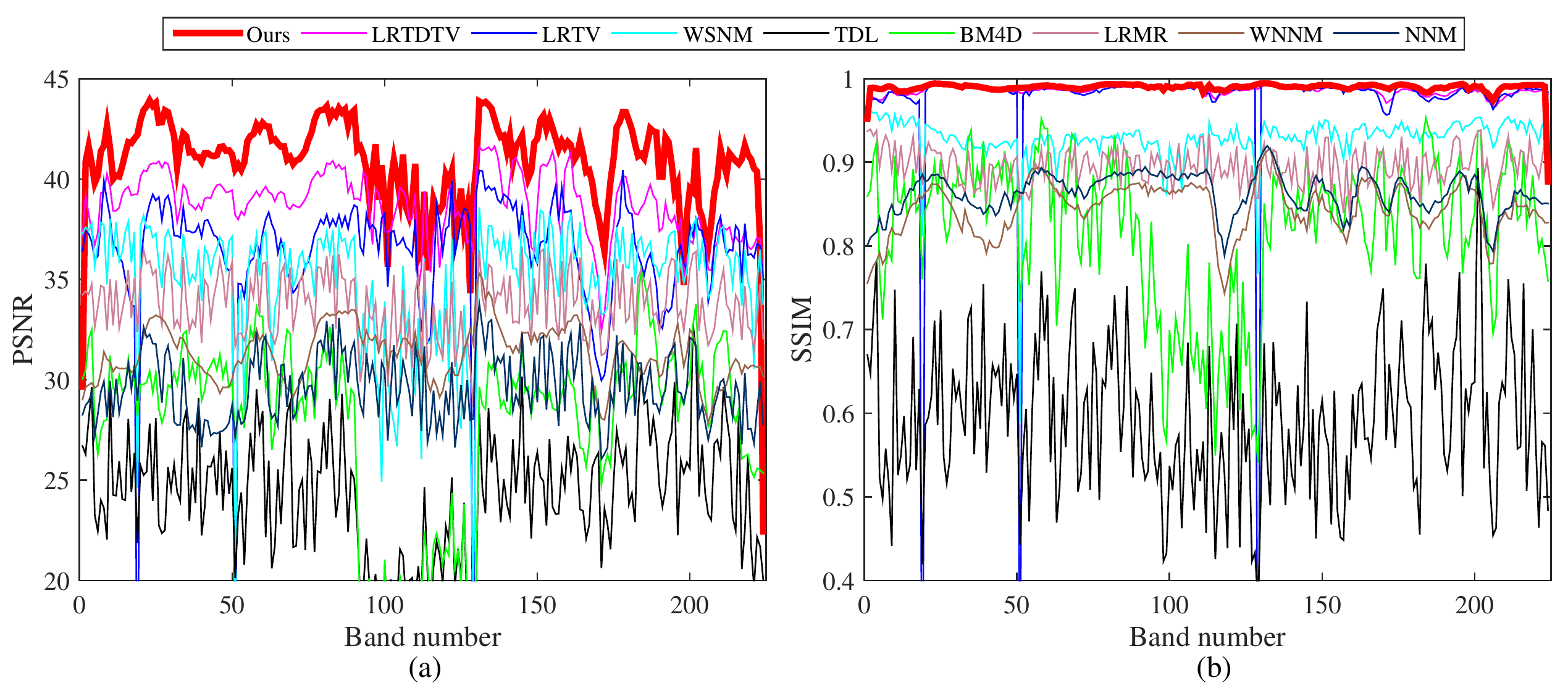}
\caption{The quality indices distribution of all competing methods calculated on the ARIVIS Indian Pines HSI in Case e. The curve of the bottom figure were enlarged in the top figure for better visualization on a computer screen. }
\label{fig_indian_qua}
\vspace{1mm}
\end{figure*}

\begin{figure*}[!]
\centering
\includegraphics[width=1\linewidth]{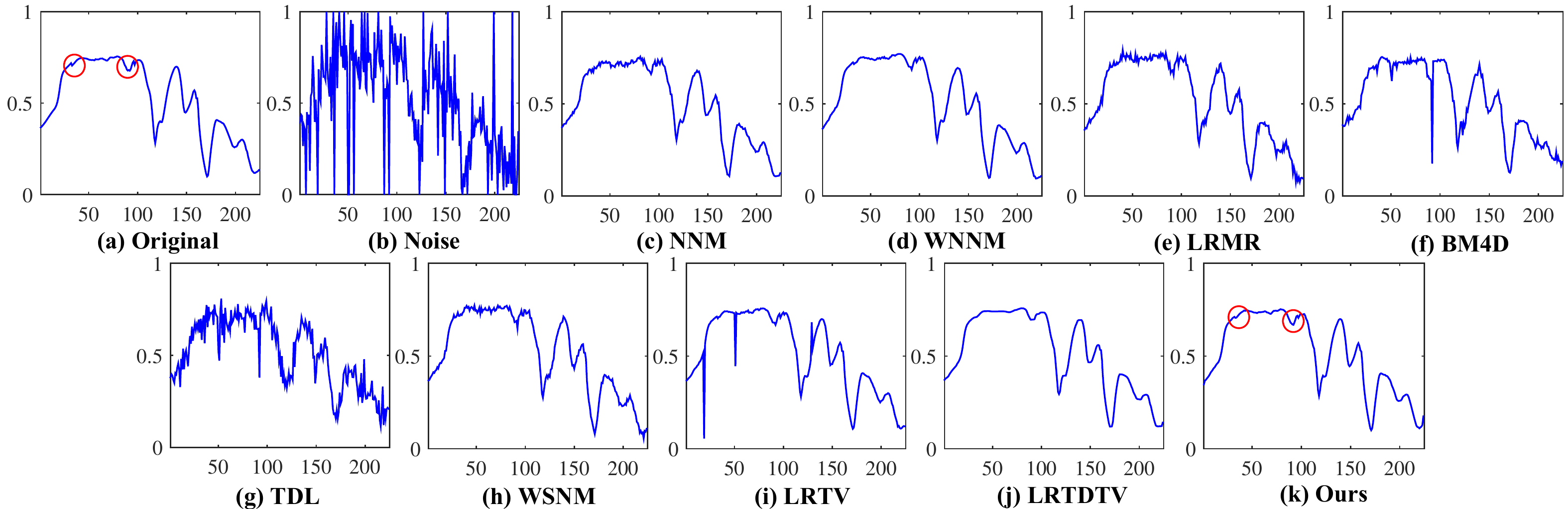}
\caption{ The spectral signatures curves of pixel (50,50) restored by all the compared methods under ARIVIS Indian Pines HSI in Case e. The detail of curve were marked by red circle for easy observation. }
\label{fig_indian_sign}
\vspace{1mm}
\end{figure*}

\begin{table*}[!]
\footnotesize
\caption{Quantitative comparison of all competing methods under different levels of noises on Indian Pines dataset. The best result for each series of experiments is highlighted in bold.}
\centering
\begin{tabular}{|c|c|c|c|c|c|c|c|c|c|c|c|}
\Xhline{2\arrayrulewidth}
$\text{Noise case}$ &$\text{Index}$ &$\text{Noise}$ &$\text{NNM}$ &$\text{WNNM} $ &$\text{LRMR} $ &$\text{BM4D}$&$\text{TDL} $ &$\text{WSNM}$ &$\text{LRTV} $ &$\text{LRTDTV}$ &$\text{Ours}$\\
\Xhline{3\arrayrulewidth}
\multirow{3}{*}{Case a} & PSNR &19.99 &30.97 &32.58 &36.20 &38.44 &38.05 &37.32 &38.68 &40.76&\textbf{42.33} \\
\cline{2-12}
& SSIM &0.3672 &0.8727 &0.8420 &0.9311 &0.9763 &0.9674 &0.9453 &0.9853 &0.9804 &\textbf{0.9910}\\
\cline{2-12}
& ERGAS &233.99 &69.02 &57.65 &36.85 &29.04 &30.72 &32.37 &28.47  &\textbf{23.02} &24.32\\
\Xhline{2\arrayrulewidth}
\multirow{3}{*}{Case b} & PSNR &19.34 &30.81 &32.46 &35.67 &35.10 &33.22 &36.12 &38.04 &40.54 &\textbf{41.97}\\
\cline{2-12}
& SSIM &0.3592 &0.8715 &0.8415 &0.9291 &0.9391 &0.8743 &0.9402 &0.9818&0.9895 &\textbf{0.9910}\\
\cline{2-12}
& ERGAS &257.88 &70.06 &58.39 &39.84 &110.40 &112.35 &44.89 &49.18&\textbf{23.44}&24.78\\
\Xhline{2\arrayrulewidth}
\multirow{3}{*}{Case c} & PSNR &13.07 &31.61 &32.36 &36.40 &28.59 &27.50 &38.14 &39.54&41.08 &\textbf{43.13}\\
\cline{2-12}
& SSIM &0.1778 &0.8889 &0.8786 &0.9345 &0.8401 &0.9076 &0.9551 &0.9866 &0.9910 &\textbf{0.9928}\\
\cline{2-12}
& ERGAS &520.53 &64.36 &59.71 &36.04 &90.29 &102.08 &29.52 &35.22 &\textbf{21.98} &22.55\\
\Xhline{2\arrayrulewidth}
\multirow{3}{*}{Case d} & PSNR &12.92 &31.45 &32.29 &35.76 &27.02 &26.54 &36.63 &38.75 &40.72&\textbf{42.68}\\
\cline{2-12}
& SSIM &0.1748 &0.8878 &0.8777 &0.9316 &0.8073 &0.8365 &0.9485 &0.9826 &0.9906 &\textbf{0.9926}\\
\cline{2-12}
& ERGAS &529.82 &65.45 &60.10 &39.99 &128.11 &120.54 &46.32 &55.44 &\textbf{22.90} &23.26\\
\Xhline{2\arrayrulewidth}
\multirow{3}{*}{Case e} & PSNR &13.80 &29.62 &31.33 &33.72 &27.95 &24.34 &35.02 &36.54&38.83&\textbf{40.80}\\
\cline{2-12}
& SSIM &0.2038 &0.8633 &0.8445 &0.8951 &0.8201 &0.5931 &0.9285 &0.9742 &0.9859&\textbf{0.9894}\\
\cline{2-12}
& ERGAS &500.68 &81.16 &66.39 &50.16 &121.95 &158.75 &51.33 &72.52 &28.66 & \textbf{28.37} \\
\Xhline{2\arrayrulewidth}
\multirow{3}{*}{Case f} & PSNR &13.73 &29.54 &29.97 &33.42 &27.53 &23.34 &33.88 &36.35 &38.63 &\textbf{40.54}\\
\cline{2-12}
& SSIM &0.2022 &0.8612 &0.8431 &0.8918 &0.8060 &0.5583 &0.9261 &0.9736 &0.9852 &\textbf{0.9888}\\
\cline{2-12}
& ERGAS &504.37 &82.18 &82.48 &52.62 &127.28 &174.28 &53.29 &72.05 &29.82 &\textbf{28.87}\\
\Xhline{2\arrayrulewidth}
\end{tabular}
\label{table:assement1}
\end{table*}

\subsection{Synthetic Data Experiments Setting }
\subsubsection{Experiments Setting}
We adopted two synthetic HSIs data. One was considered by \cite{He2016}, generated from the Indian Pines dataset\footnote{\url{https://engineering.purdue.edu/~biehl/MultiSpec/hyperspectral.html}}. The size of the Indian pines synthetic HSI is $145\times 145\times 224$, with evident local smoothness property. The other one is ARIVIS DCMall HSI dataset, whose size is  $200\times 200\times 160$, with complex structure and texture information. To simulate noisy HSI data, we added six different types of noises to both original HSI data. The details are listed as follows:

\textbf{Case a}: Zero-mean Gaussian noise is added to each band. The variances of the Gaussian noise is 0.1.

\textbf{Case b}: Gaussian noise is added to each band just as Case a. In addition, some deadlines are added from band 91 to band 130, with the number of stripes randomly selected from 3 to 10, and the width of the stripes randomly generated from 1 to 3.

\textbf{Case c}: The mixture of zero-mean Gaussian noise and impulse noise are added to each band. The variance of the Gaussian noise and percentage of the impulse noise are set as 0.075 and 0.15, respectively.

\textbf{Case d}: The mixture of Gaussian and impulse noises are added like Case c, and the deadlines are added as Case b.

\textbf{Case e}: Zero-mean Gaussian noises with different variances and impulse noise with different percentage are added to each band, with the variance value being randomly selected from 0 to 0.2, and percentage value being randomly selected from 0 to 0.2. In addition, the deadlines are added to some bands as Case d.

\textbf{Case f}: The mixture of Gaussian noise, impulse noise, and deadlines are added as Case e did. In addition, some stripes are added from band 161 to band 190 in Indian pines, and some stripes are added from band 141 to band 160 in DC mall,  with the number of stripes being randomly selected from 20 to 40.

As for parameter settings, there are three parameters in the model that need to be set. They are the desired rank $r$, the noise coefficient $\lambda$, and the sparsity basis coefficient $\tau$. Since we used the $\ell_1$ norm to model the noise, inspired by the parameter settings of the RPCA, we set $\lambda=\frac{1}{\sqrt{hw}}$, and set the sparse basis coefficient as $\tau=C\times \sqrt{hw}$. Note that we should manual adjust the $C$ value, because  the sparse basis coefficient and data is relevant, different data have different coefficient. In our synthetic data experiments, the $C$ value of Indian pines was set as 0.004, and that of DC mall was set as 0.0004. The rank is set as 13 in all Indian pines experiments, and 6 in DC mall experiments.

\begin{table*}[!]
\footnotesize
\caption{Quantitative comparison of all competing methods under different levels of noises on DCMall dataset. }
\centering
\begin{tabular}{|c|c|c|c|c|c|c|c|c|c|c|c|}
\Xhline{2\arrayrulewidth}
$\text{Noise case}$ &$\text{Index}$ &$\text{Noise}$ &$\text{NNM}$ &$\text{WNNM} $ &$\text{LRMR} $ &$\text{BM4D}$&$\text{TDL} $ &$\text{WSNM}$ &$\text{LRTV} $ &$\text{LRTDTV}$ &$\text{Ours}$\\
\Xhline{3\arrayrulewidth}
\multirow{3}{*}{Case a} & PSNR &20.00 &32.03 &33.00 &33.77 &32.39 &35.53 &31.68 &35.16 &35.23&\textbf{36.56} \\
\cline{2-12}
& SSIM &0.5147 &0.9584 &0.9569 &0.9584 &0.9308 &0.9710 &0.9429 &0.9639 &0.9636 &\textbf{0.9761}\\
\cline{2-12}
& ERGAS &375.83 &90.87 &83.40 &74.33 &86.67 &60.37 &95.76 &62.98  &61.83 &\textbf{55.74}\\
\Xhline{2\arrayrulewidth}
\multirow{3}{*}{Case b} & PSNR &19.87 &31.87 &32.93 &33.55 &31.49 &34.51 &31.63 &34.18 &34.33 &\textbf{35.57}\\
\cline{2-12}
& SSIM &0.5092 &0.9581 &0.9565 &0.9572 &0.9209 &0.9650 &0.9423 &0.9531&0.9575 &\textbf{0.9698}\\
\cline{2-12}
& ERGAS &381.50 &92.38 &83.92 &76.63 &101.86 &71.58 &96.51 &76.38&68.90&\textbf{61.19}\\
\Xhline{2\arrayrulewidth}
\multirow{3}{*}{Case c} & PSNR &12.39 &32.70 &33.49 &31.77 &23.63 &23.73 &31.92 &34.46&35.41&\textbf{36.35}\\
\cline{2-12}
& SSIM &0.2207 &0.9651 &0.9632 &0.9390 &0.6869 &0.8010 &0.9465 &0.9563 &0.9672 &\textbf{0.9751}\\
\cline{2-12}
& ERGAS &925.41 &84.53 &79.64 &94.28 &251.78 &250.93 &93.28 &91.02 &61.54 &\textbf{56.02}\\
\Xhline{2\arrayrulewidth}
\multirow{3}{*}{Case d} & PSNR &12.37 &32.52 &33.40 &31.61 &23.55 &23.80 &31.87 &34.33&35.23&\textbf{36.08}\\
\cline{2-12}
& SSIM &0.2187 &0.9646 &0.9627 &0.9369 &0.6814 &0.8003 &0.9461 &0.9548 &0.9658 &\textbf{0.9732}\\
\cline{2-12}
& ERGAS &926.52 &86.13 &80.39 &96.41 &252.44 &246.69 &94.05 &97.59 &63.98 &\textbf{58.49}\\
\Xhline{2\arrayrulewidth}
\multirow{3}{*}{Case e} & PSNR &13.70 &31.36 &32.72 &30.83 &24.53 &25.64 &31.57 &33.37&34.25&\textbf{35.42}\\
\cline{2-12}
& SSIM &0.2698 &0.9559 &0.9551 &0.9256 &0.7094 &0.8420 &0.9413 &0.9460 &0.9580 &\textbf{0.9694}\\
\cline{2-12}
& ERGAS &793.92 &98.30 &85.48 &106.43 &228.07&203.72 &96.44 &99.25 &70.40 &\textbf{62.27}\\
\Xhline{2\arrayrulewidth}
\multirow{3}{*}{Case f} & PSNR &13.67 &31.45 &32.76 &30.80 &24.52 &25.56 &31.47 &33.43 &34.34&\textbf{35.23}\\
\cline{2-12}
& SSIM &0.2721 &0.9566 &0.9561 &0.9252 &0.7100 &0.8383 &0.9410 &0.9446 &0.9575 &\textbf{0.9672}\\
\cline{2-12}
& ERGAS &812.13 &98.20 &85.78 &108.08 &230.84 &209.68 &98.38&115.93 &71.13 &\textbf{63.51}\\
\Xhline{2\arrayrulewidth}
\end{tabular}
\label{table:assement2}
\end{table*}

\begin{figure*}[!]
\centering
\includegraphics[width=1\linewidth]{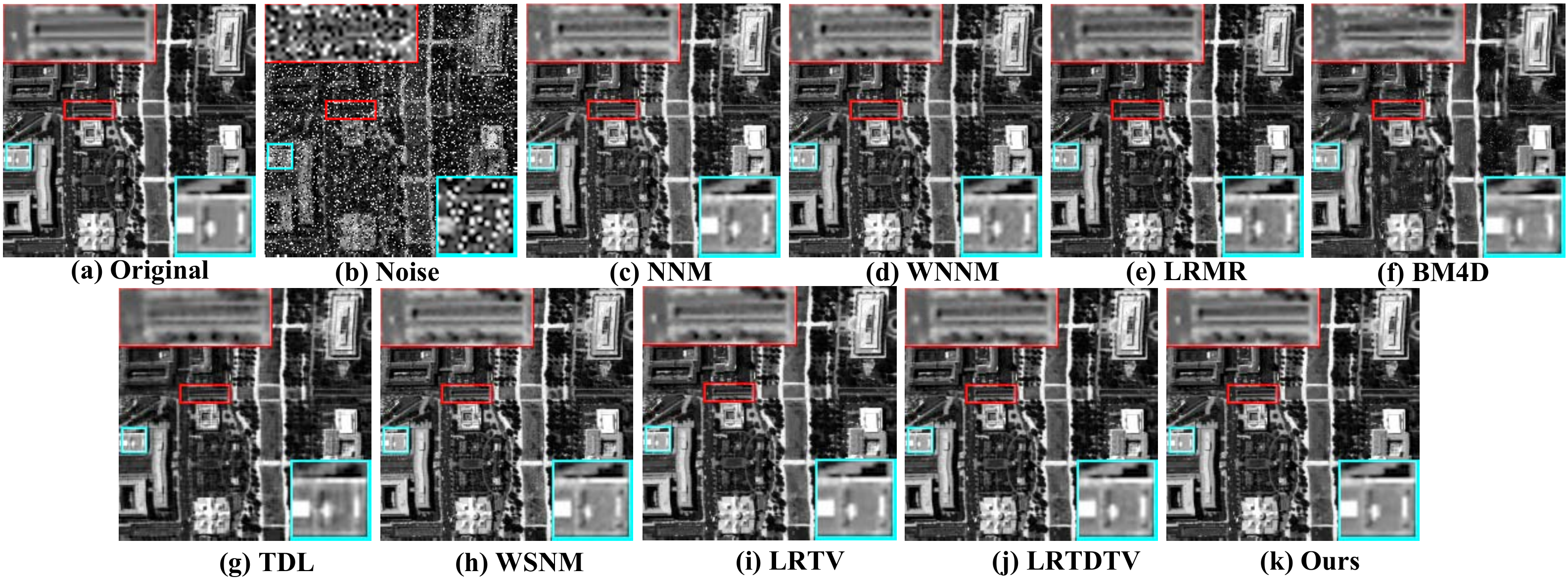}
\caption{The restoration performance of all competing methods on the band 30 of the DC mall HSI in Case c. }
\label{fig_dcmall30}
\end{figure*}

\begin{figure*}[!]
\centering
\includegraphics[width=1\linewidth]{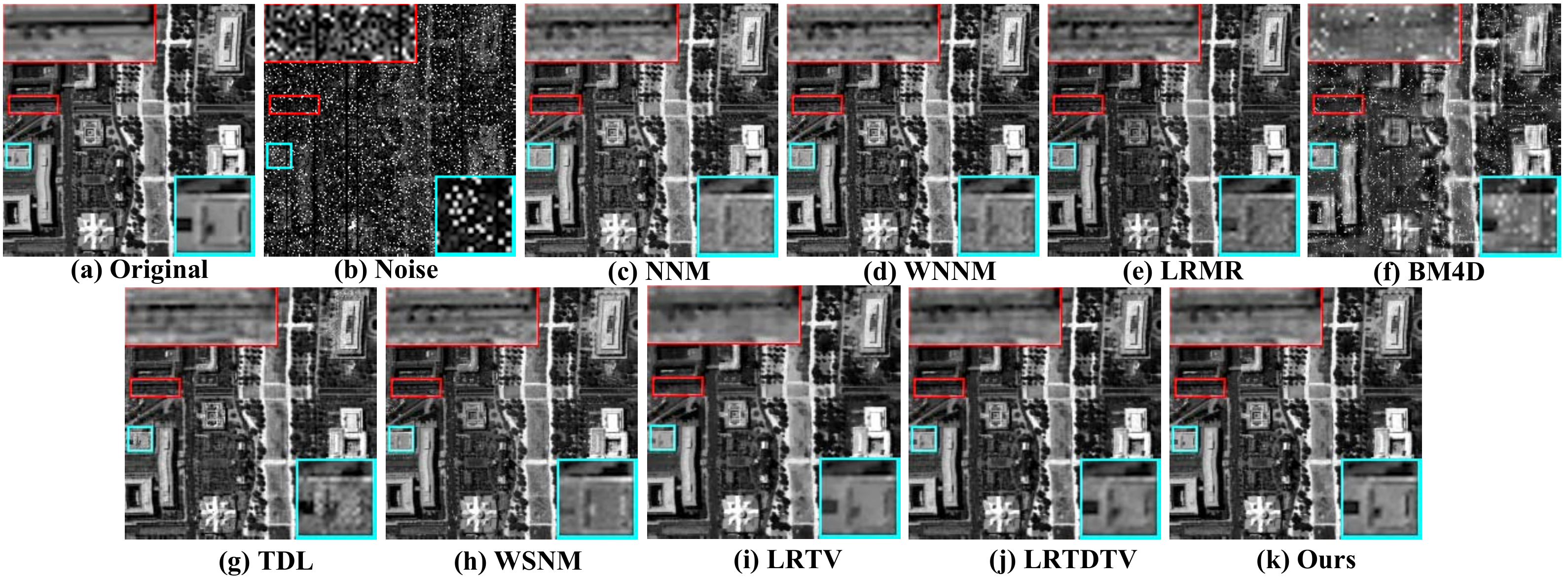}
\caption{The restoration performance of all competing methods on the band 120 of the DC mall in Case e.}
\label{fig_dcmall120}
\end{figure*}

\begin{figure*}[!]
\centering
\includegraphics[width=1\linewidth]{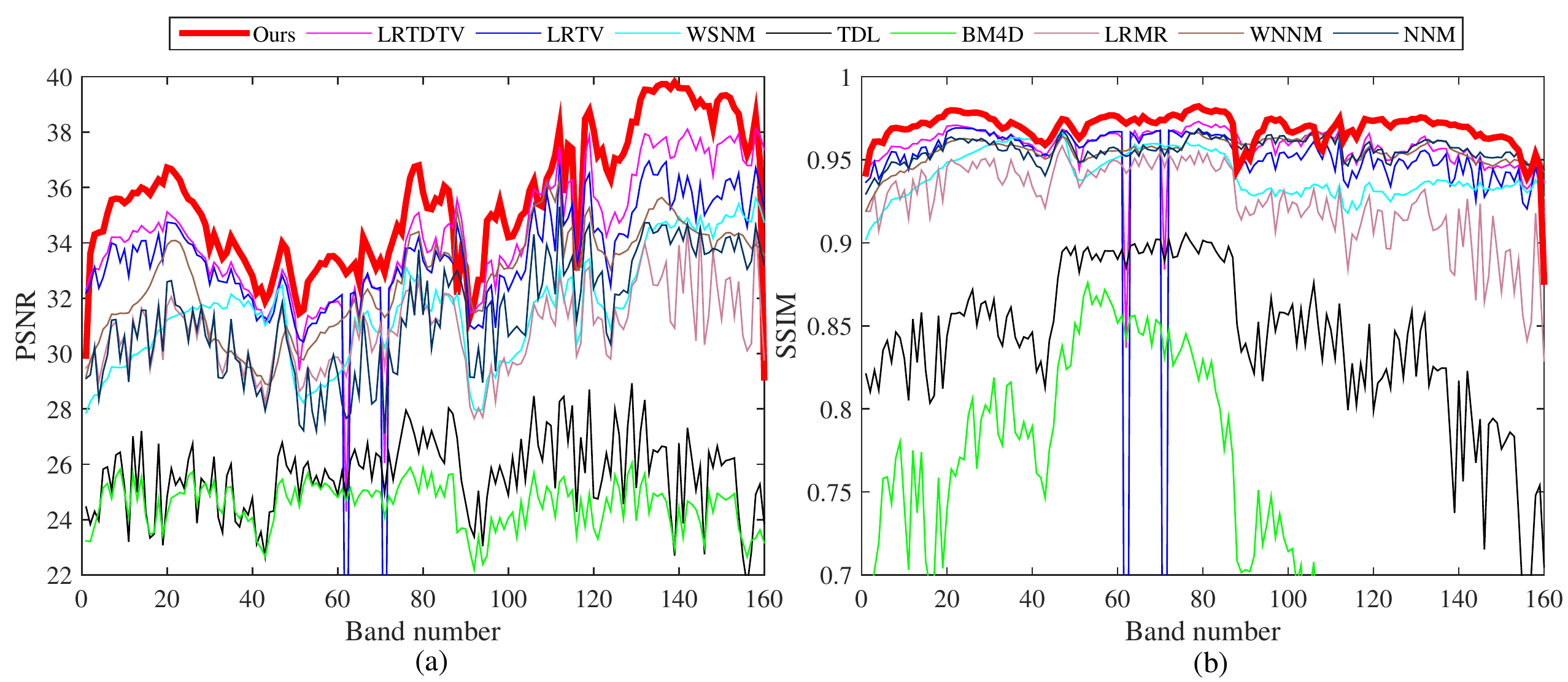}
\caption{The quality indices distribution of all competing methods  of the DC mall HSI in Case e. }
\label{fig_dcmall_qua}
\end{figure*}

\begin{figure*}[!]
\centering
\includegraphics[scale=0.55]{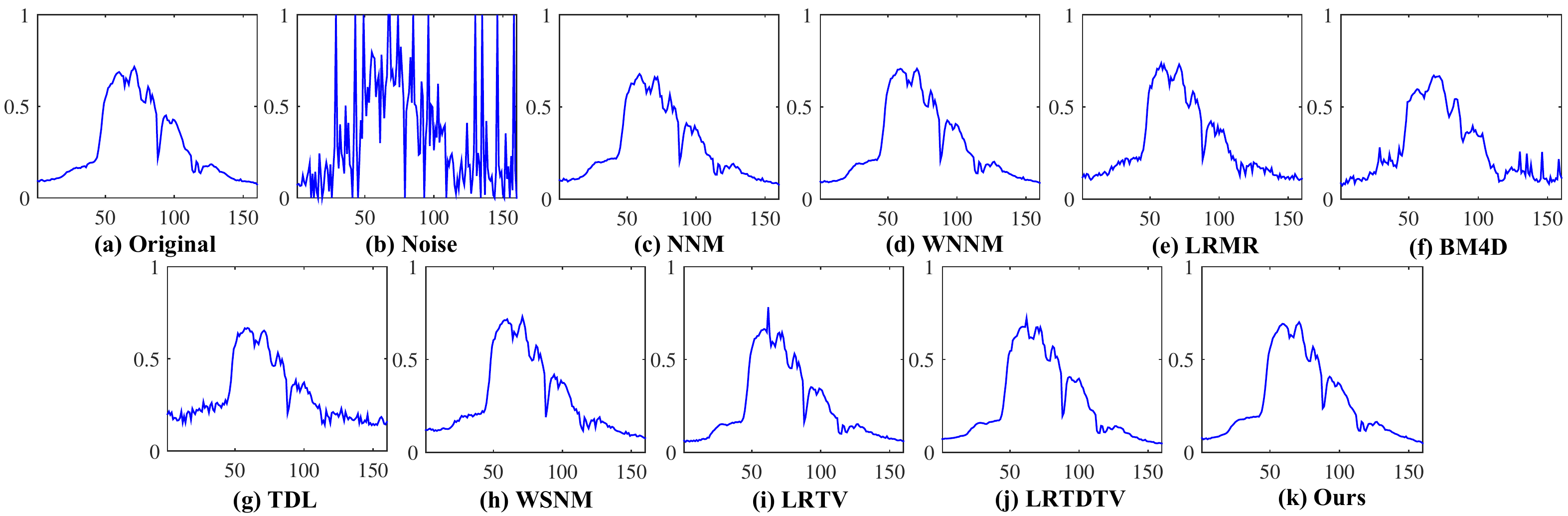}
\caption{ The spectral signatures curves of pixel (100,100) restored by all the compared methods under DC mall HSI in Case e. }
\label{fig_dcmall_sign}
\end{figure*}
\subsubsection{The Synthetic Indian pines HSI experiments}

In terms of visual quality, two representative bands of restored HSIs in two typical Case c and Case e  obtained by different methods are shown.  Fig. \ref{fig_indian220} shows the restoration
results of band 220 under Case c,  and Fig. \ref{fig_indian120} shows the restoration results of band 120 under Case e. Compared with other methods, it can be easily observed that the proposed method recovers more meaningful details including textures and edges, and meanwhile produces better reconstruction results in smooth regions with higher PSNR and SSIM values. And the mean overall quantitative assessment results by all competing denoising methods are listed in Table. \ref{table:assement1}, and the distribution of quality indices under Case e is depicted in Fig. \ref{fig_indian_qua} to show the each band's restoration performance. In order to clearly depict the top of curve in the (a-2) and (b-2) pictures, we enlarge this area in the (a-1) and (b-1) figures. From the figure, we can easily see that the evaluation index curve of E-3DTV  is higher than other methods, which verifies the better recovery performance of the E-3DTV in almost all bands.

To further compare the performance of all competing methods, it would be necessary to show the spectral signatures before and after restoration. As such, Fig. \ref{fig_indian_sign} shows one simple example namely the spectral signatures of pixel (50, 50) in Case e. It can be clearly seen that the LRTDTV as well as E-3DTV  methods show the better restoration results than other methods. Especially, the E-3DTV method can preserve some detailed changes of spectral curve. Combined with the quality indices values shown in Table \ref{table:assement1}, our proposed E-3DTV method attains the best restoration results among all compared methods.

\subsubsection{The Synthetic DC Mall HSI Experiments}

\begin{figure*}[!]
\centering
\includegraphics[width=0.97\linewidth]{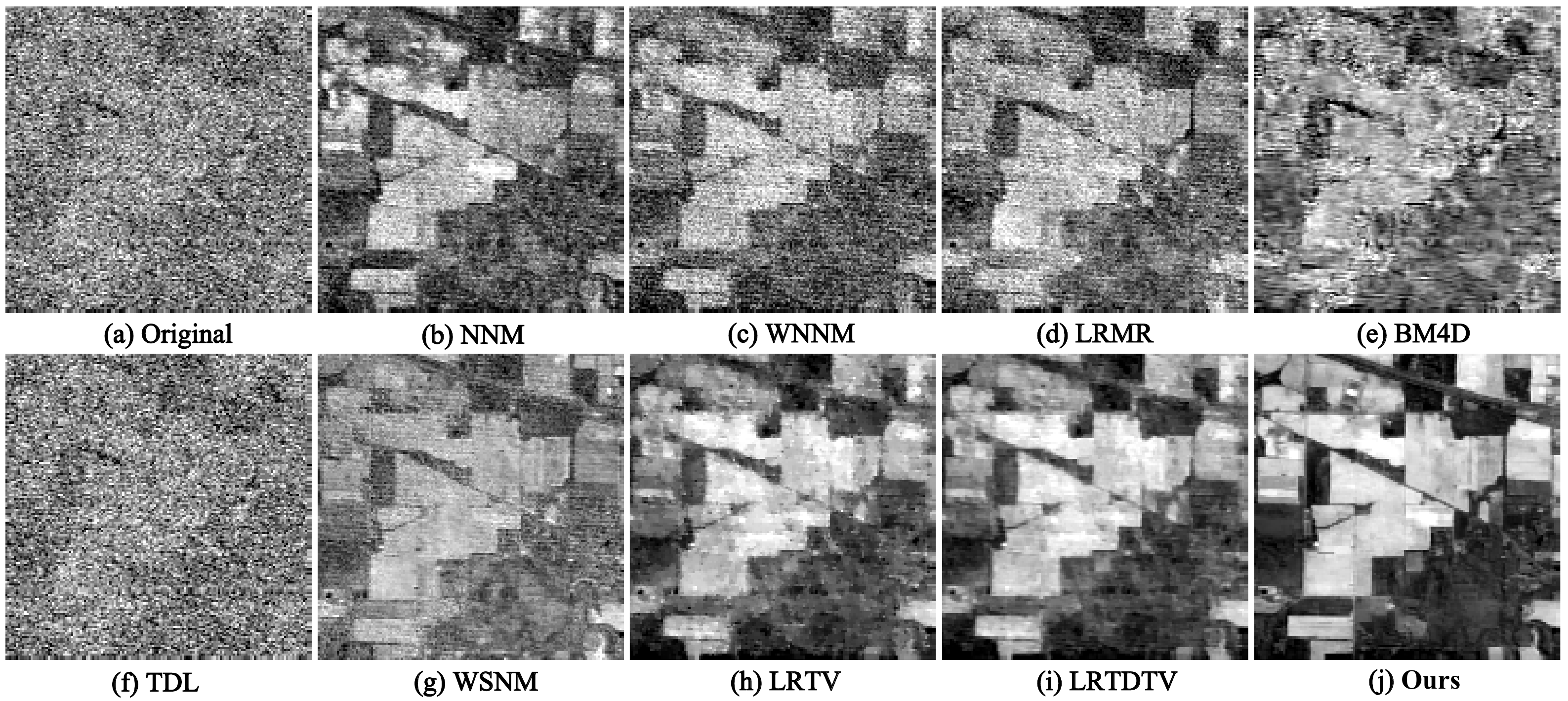}
\caption{The restoration performance of all competing methods on the band 108 of real Indian Pines HSI. }
\label{fig_indian108}
\end{figure*}

\begin{figure*}[!]
\centering
\includegraphics[width=0.97\linewidth]{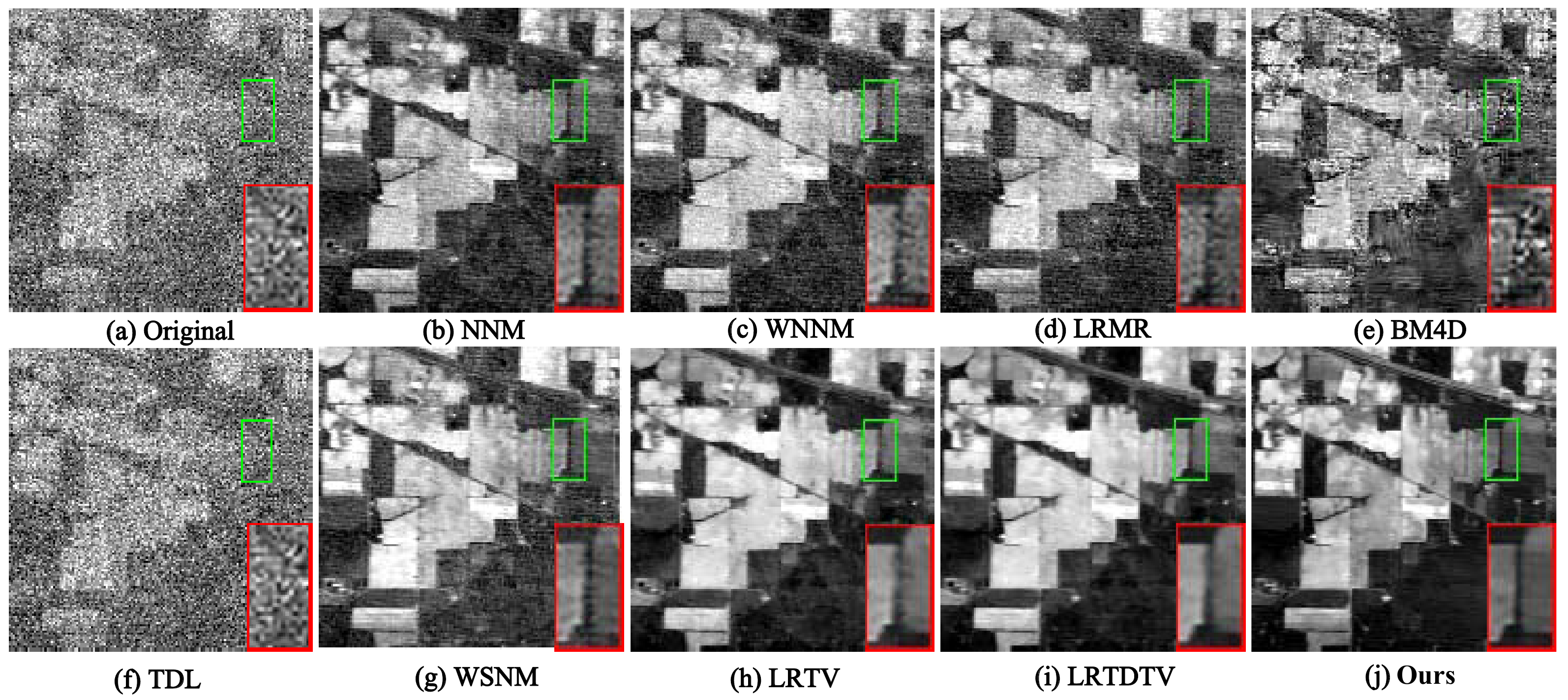}
\caption{The restoration performance of all competing methods on the band 220 of real Indian Pines HSI. }
\label{fig_indian220}
\end{figure*}

\begin{figure*}[!]
\centering
\includegraphics[width=0.97\linewidth]{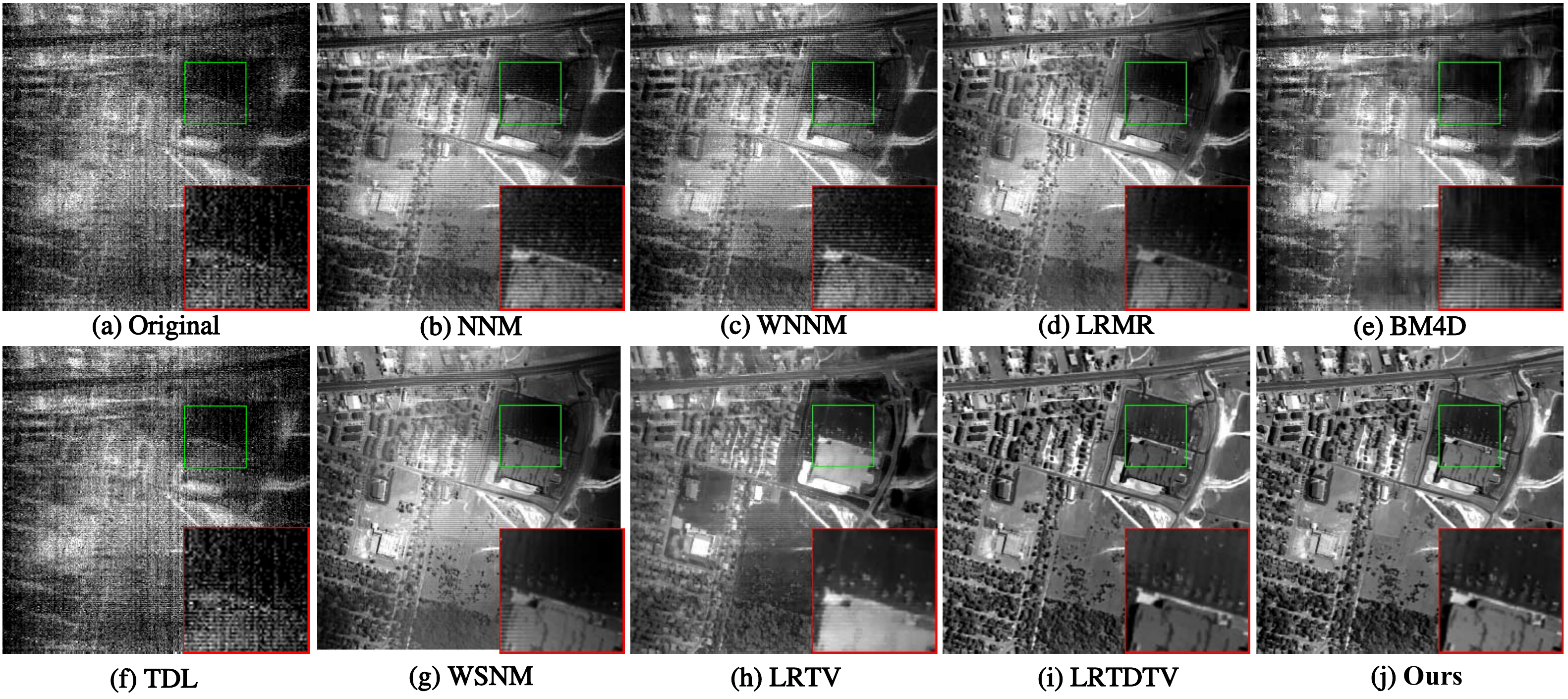}
\caption{The restoration performance of all competing methods on the band 108 of real urban HSI.}
\label{fig_urban_108}
\end{figure*}

\begin{figure*}[!]
\centering
\includegraphics[width=0.97\linewidth]{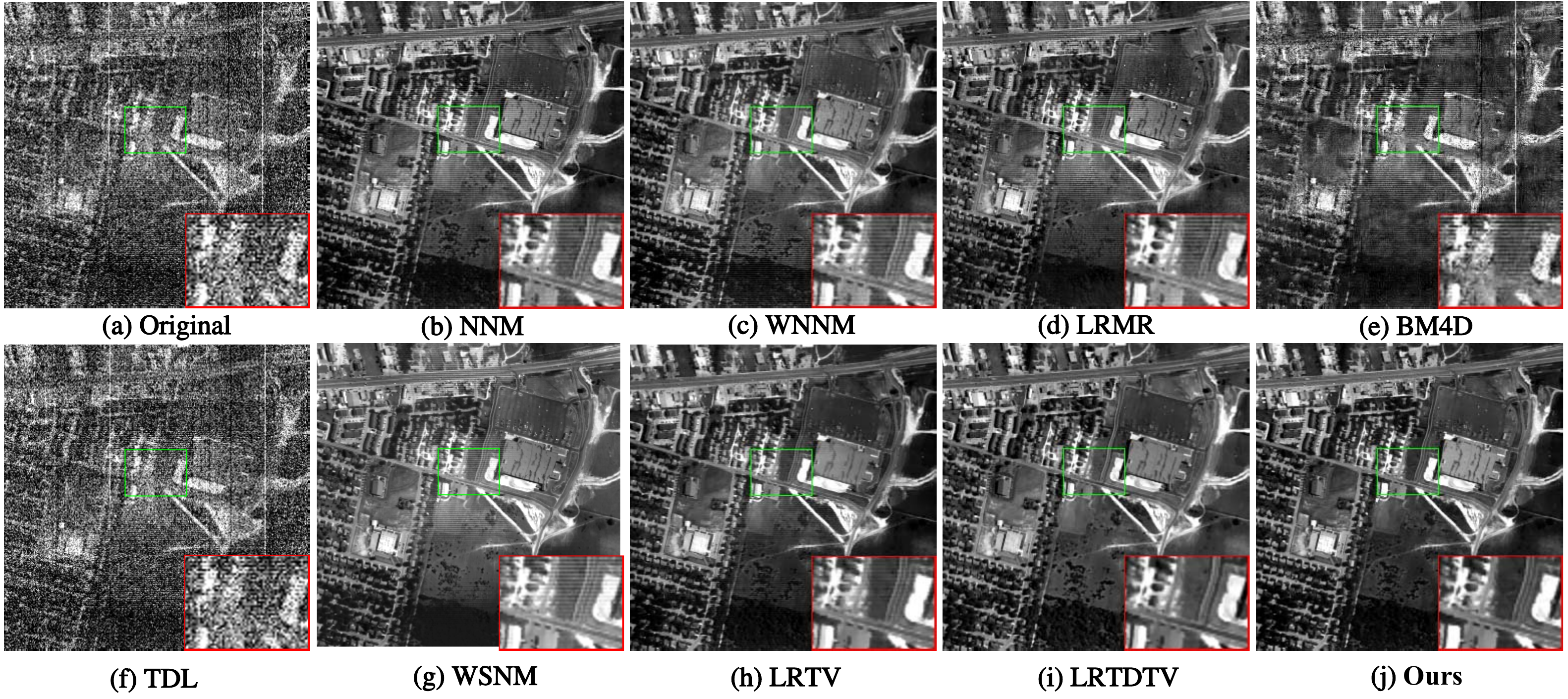}
\caption{The restoration performance of all competing methods on the band 207 of real urban HSI.}
\label{fig_urban_207}
\end{figure*}

\begin{figure*}[!]
\centering
\includegraphics[width=0.985\linewidth]{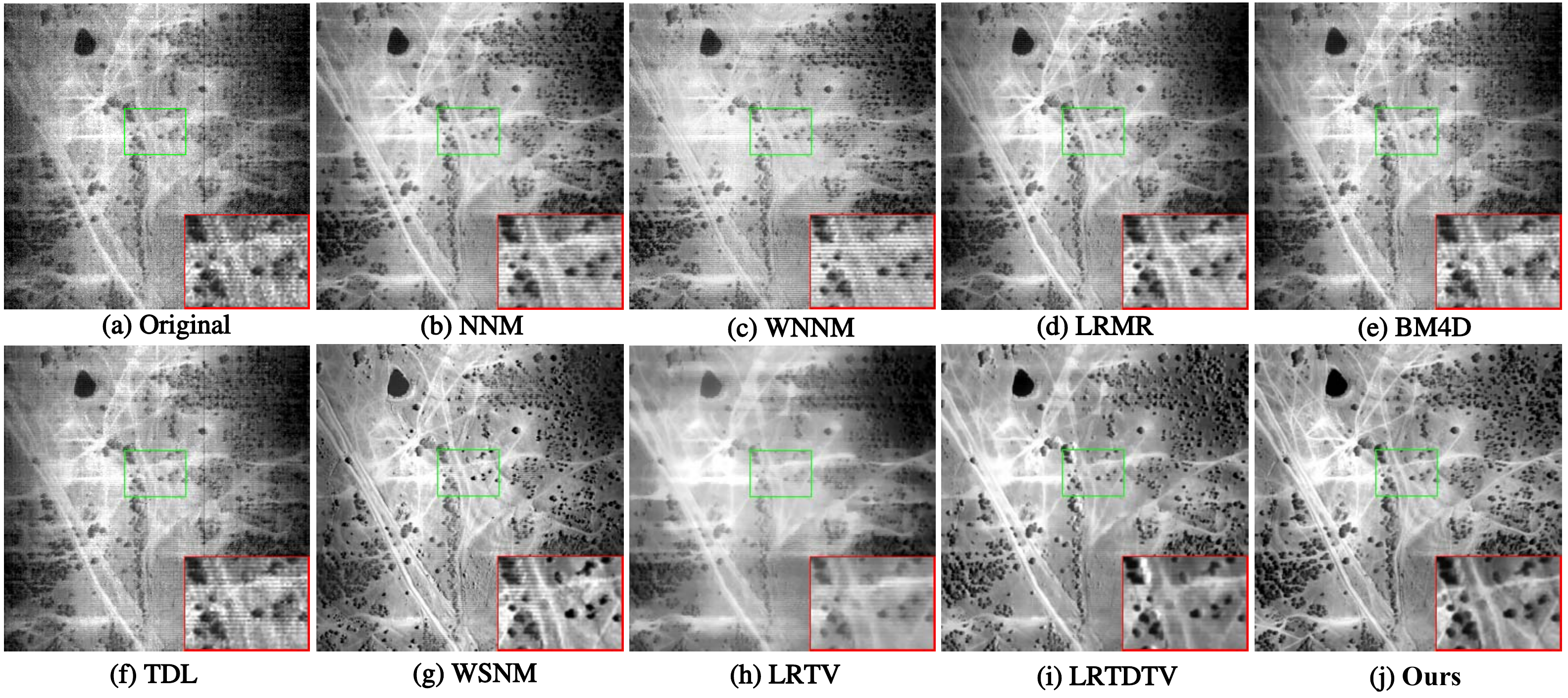}
\caption{The restoration performance of all competing methods on the band 104 of real terrian HSI. }
\label{fig_teerian_104}
\end{figure*}

\begin{figure*}[!]
\centering
\includegraphics[width=0.985\linewidth]{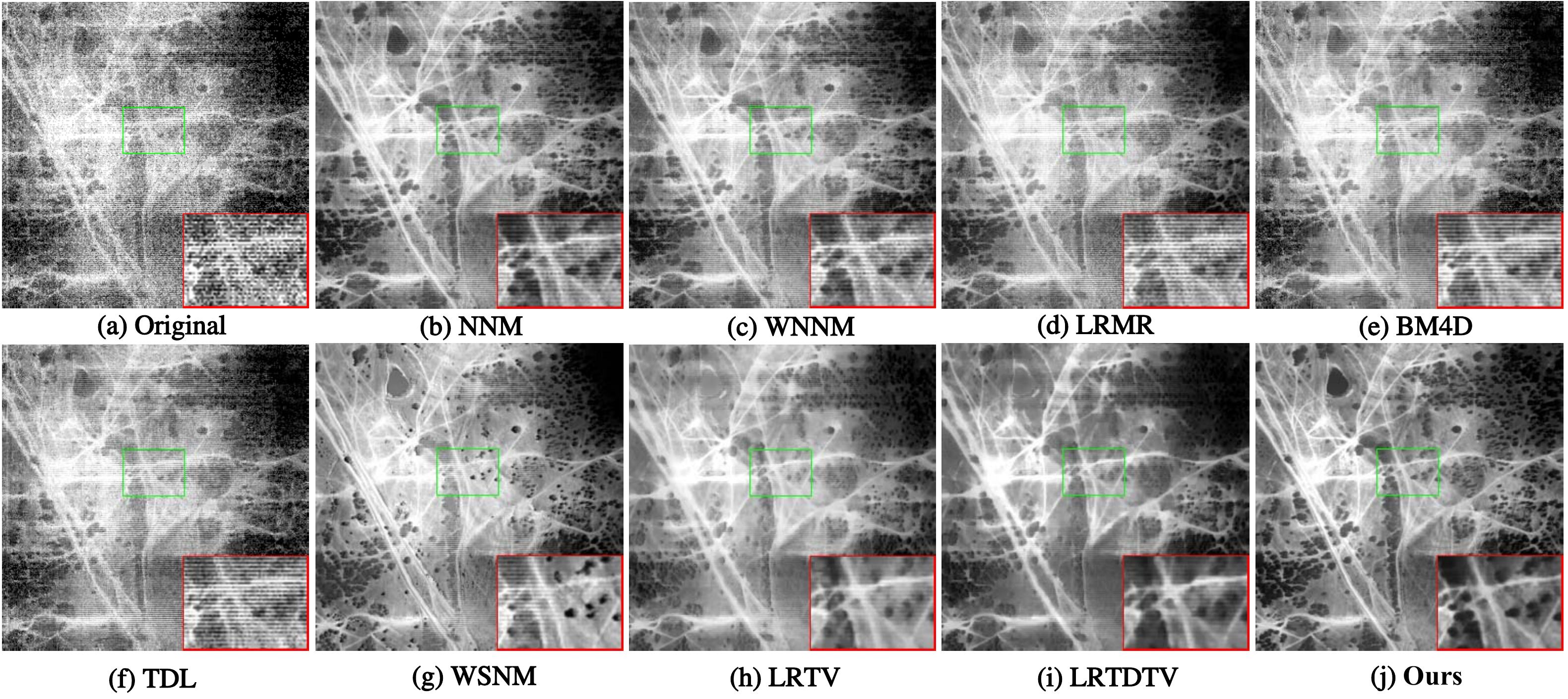}
\caption{The restoration performance of all competing methods on the band 207 of real terrian HSI.}
\label{fig_teerian_151}
\end{figure*}

\begin{figure*}[!]
\centering
\includegraphics[width=1\linewidth]{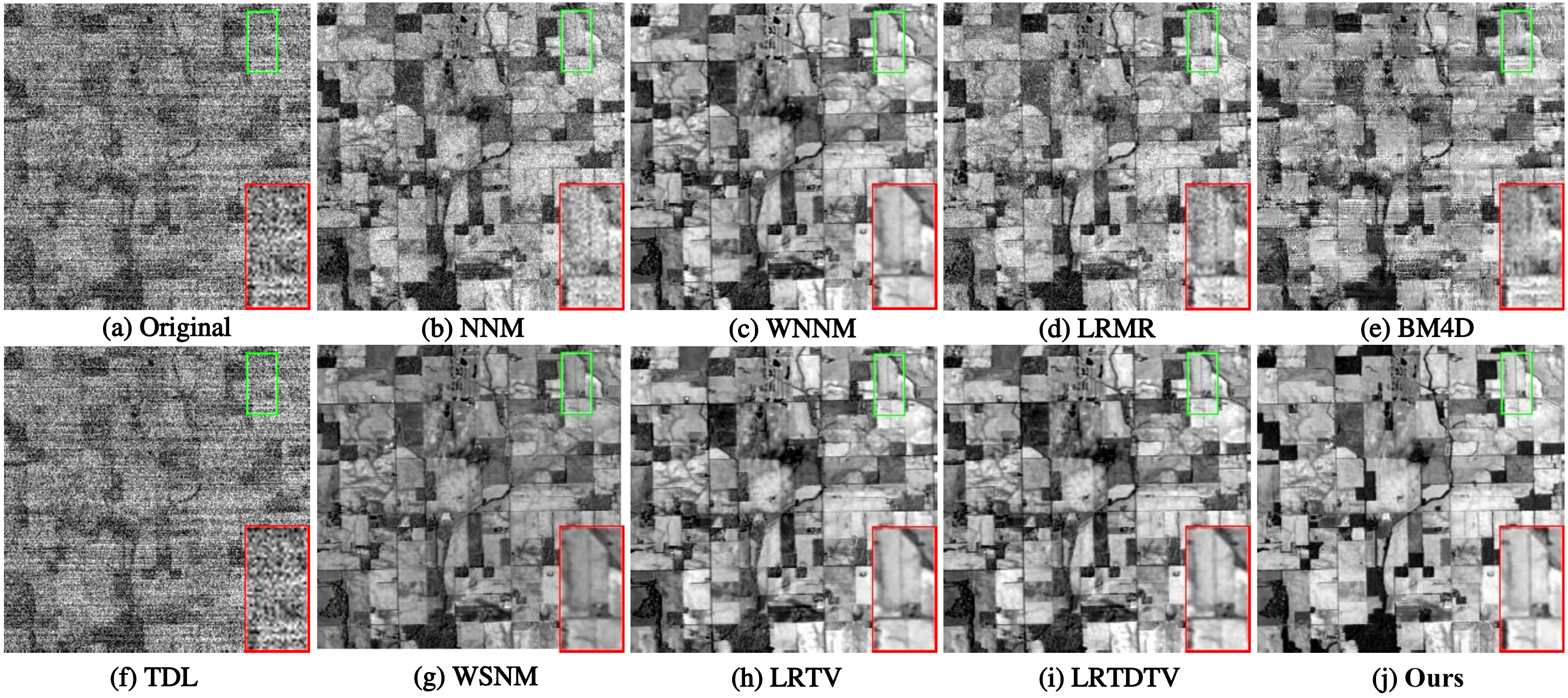}
\caption{The restoration performance of all competing methods on the band 108 of real lowal altitude HSI.}
\label{fig_lowal_109}
\end{figure*}

\begin{figure*}[!]
\centering
\includegraphics[width=1\linewidth]{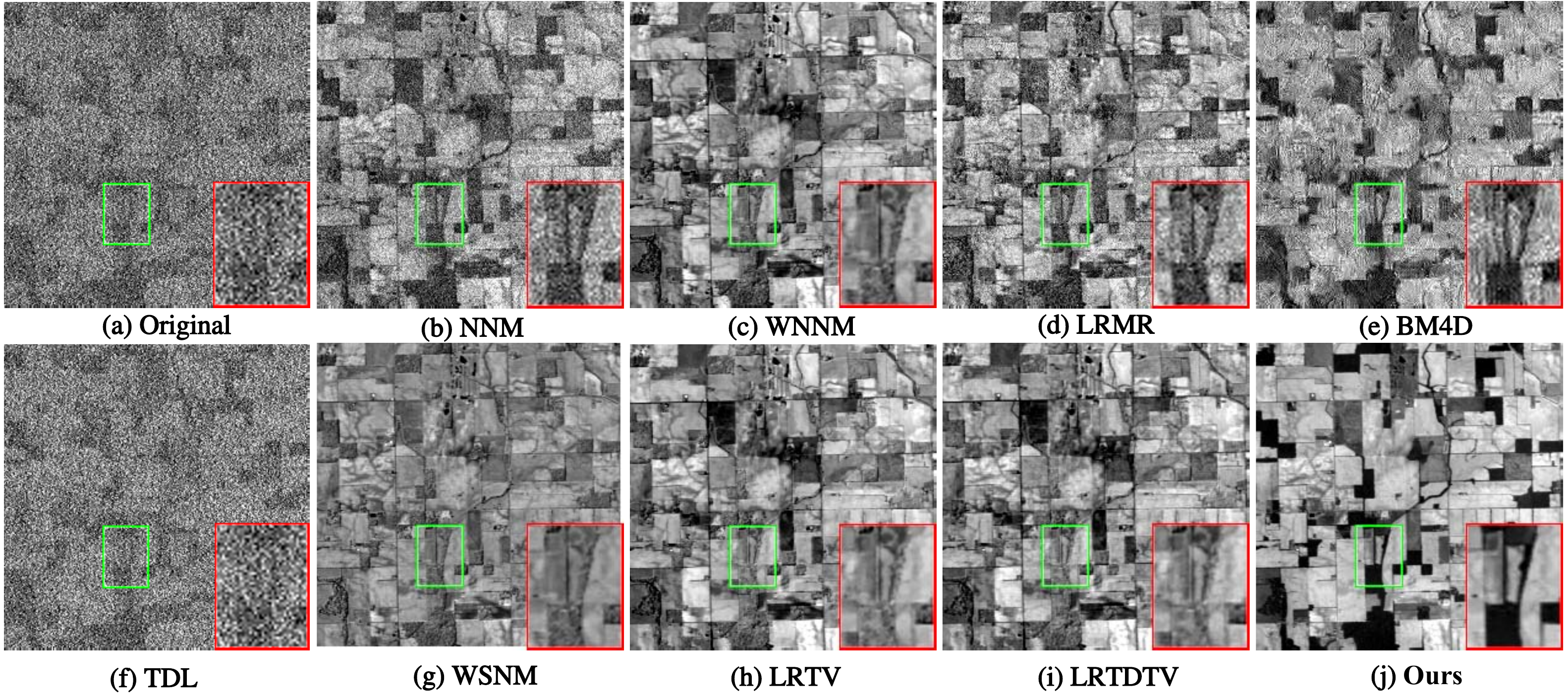}
\caption{The restoration performance of all competing methods on the band 207 of real lowal altitude HSI.}
\label{fig_lowal_162}
\end{figure*}

We then show the performance of all competing methods on the DC mall data, which have more complex structure than previous one. We also provide the two visual restoration images in  Figs. \ref{fig_dcmall30} and \ref{fig_dcmall120}, respectively. The average of all quantitative assessment results are provided in Table. \ref{table:assement2}, the distribution of quality indices under Case e is shown in Fig. \ref{fig_dcmall_qua}, and the signature restoration of  pixel (100, 100) in case e are provided in Fig. \ref{fig_dcmall_sign}.

From the table and figures, it can be easily seen that our proposed E-3DTV method is also of superior performance among other competing methods both in quantity and visualization. The effectiveness of the proposed E-3DTV regularizer is further substantiated.

%Through the restoration effect of the above two simulation data sets, we can see that our proposed a priori is a universal prior for preserving information of HSI.
%-------------------------------------------------------------------------
\subsection{Real Data Experiments}
Four typical real-world HSI datasets were further used for experiments. These datasets include ARIVIS Indian\footnote{\url{https://engineering.purdue.edu/~biehl/MultiSpec/hyperspectral.html}}, ARIVIS Low altitude\footnote{\url{https://engineering.purdue.edu/~biehl/MultiSpec/hyperspectral.html}}, HYDICE Urbanpart\footnote{\url{http://www.tec.army.mil/hypercube}} and Terrain\footnote{\url{http://www.tec.army.mil/hypercube}}. These real HSIs all contain serious pollutions including water pollution, deadlines, strips and sparse noise, and their shapes and structures are very different. It is thus a big challenge to recover the clean HSIs from them. In order to comprehensively display the restoration effects of all methods, we select two representative bands to show the visual restoration effects for all competing methods.
%Complete visual comparisons are provided as a demo video in supplementary material.

From Figs. \ref{fig_indian108}-\ref{fig_lowal_162}, it can be clearly observed that the methods based on utilizing the local smoothness prior and low rank prior such as LRTV and LRTDTV, together with the proposed E-3DTV method, perform better than other competing methods.
%These phenomenon show that the effect of the the methods based on the low rank + spatial smoothness prior. Therefore, it's normal to achieve the great denosing effects because the SCS method inherits the properties of low rank and local smoothness.
Specifically, the E-3DTV method can better recover the overall structure information of HSI and suppress the noises as shown in Figs. \ref{fig_indian108}, and get the best restoration result in preserving the local edge/texture information. The E-3DTV method can also be seen to better remove the structure noises (such as deadlines and stripes) and preserve the local texture of the HSIs from the Figs. \ref{fig_urban_108} to \ref{fig_teerian_151}. The proposed method can better remove the artificial shadows and overlap to effectively avoid false error information caused by low-rank frames as depicted in Fig. \ref{fig_lowal_162}.

In all, we can see that the proposed E-3DTV method clearly outperforms all other competing methods in both its better noise suppressing effect and its finer meaningful structure preservation capabilities. This substitutes the usefulness of the proposed E-3DTV regularization on such task in real complicated cases.

\section{E-3DTV Method for HSI Compressed Sensing}
\subsection{The E-3DTV Model}
The compressed sensing task is to reconstruct the HSI  from a small number of compressed measurements $y$ with global information. Then the compressed measurements $y\in\mathbb{R}^{m}$ can be obtained by
\begin{equation}
\label{cs_function}
y=\mathbf{\Psi}\left(\Z\right),
\end{equation}
where $\Z\in\mathbb{R}^{hw\times s}$ denote the to-be-reconstructed HSI. $\Z$ is generally considered to contain the clean HSI term and noise term mixed in the signals during the acquisition. Denote the clean HSI and noise as $\X$, $\E$, respectively, and then we can get $\Z=\X+\E$. The compressive operator $ \mathbf{\Psi} $ can be instantiated as $ \mathbf{\Psi}=\mathbf{D}\cdot\mathbf{H}\cdot\mathbf{P}$, where $\mathbf{H}$ is a random permutation matrix, $\mathbf{P}$ is the WalshHadamard transform, and $\mathbf{D}$ is a random down sampling operator. Such a compressive operator has been shown to satisfy the restricted isometry property and successfully used in dealing with various compressed sensing problems \cite{Jiang2017Surveillance,Gogna2015Split,CaoW2016H}.

%It is evident that recovering X based on (1) is an ill-posed inverse problem, and thus, various priors need to be discovered to make this problem well posed. So we add our proposed GSC prior to this ill problem. Since noise is mixed during the acquisition of signals, the data needed to be reconstructed may not hold the same extension property of the clean data. Considering the actual transmission, some seriously polluted bands little or no information. In order to save the transmission resources, these bands can be thrown away at transmission. So in here, we assume that the noise obeys a Gaussian distribution. In this way,  the HSI image to be reconstructed can be split into two terms, one is the clean data term, the another is the noise term, then the prior knowledge can be added into the clean data term.  Based on the aforementioned discussion, in this study, the GCS reconstruction model is to estimate the desired 3-order HSI images $\mathcal{X}$ from the measured vector $y$ , which can be expressed as follows:

Recovering $\Z$ based on (\ref{cs_function}) is an ill-posed inverse problem, and thus prior information should necessarily be exploited. Here we simply adopted the proposed E-3DTV regularization as the unique prior term in the model. In addition, considering the actual transmission, some seriously polluted bands have little or no information. In order to save the transmission resources, these bands can be thrown away at transmission. So here, we assume that the noise obeys a Gaussian distribution as previous works did \cite{Wang2017Compressive}, and a $\ell_2$ norm term is adopted on the noise term. The proposed HSI compressed sensing model is then formulated as follows:
\begin{equation}\label{CS_model0}
\begin{split}
&\min_{{\X},{\E}} \tau\|\X\|_{\scriptsize{\mbox{E-3DTV}}}+\frac{1}{2}\|{\E}\|_2 \\
&~{\rm s.t.}~ y=\mathbf{\Psi}\left({\X}+{\E}\right),\\
\end{split}
\end{equation}
where $\mathbf{\tau}$ is the trade-off parameter for controlling the sparsity of the gradient map of the HSI and noise regularization parameter. By adopting (\ref{Reg2}), and introducing a auxiliary matrix $\Z = \X+\E$, we can rewrite (\ref{CS_model0}) as:
\begin{equation}
\begin{split}\label{cs_model}
&\min_{{\X},{\Z},{\E},\mathbf{U}_n,\mathbf{V}_n} \tau\sum_{n=1}^3\|\mathbf{U}_n\|_1+\frac{1}{2}\|{\E}\|_2 \\
&~{\rm s.t.}~ y=\mathbf{\Psi}\left({\Z}\right),{\Z}=\X+{\E}\\
&\qquad \nabla_n{\X}=\mathbf{U}_n\mathbf{V}_n^T, \mathbf{V}_n^T\mathbf{V}_n=\mathbf{I}, n=1,2,3.\\
\end{split}
\end{equation}
%Then we can solve this model by a ADMM  strategy.

\subsection{The ADMM Algorithm}

The proposed optimization problem (\ref{cs_model}) can be readily solved by the ADMM strategy. Firstly, the following augmented Lagrangian function is required to be minimized:
\begin{shrinkeq}{-2ex}
\begin{equation}
\small
\begin{split}\label{cs_lag}
&\mathcal{L}(\Z,\X,\E,\mathbf{U}_n,\mathbf{V}_n,\mathbf{M}_n,\mathbf{\mathbf{\Gamma}}_{1},\mathbf{\mathbf{\Gamma}}_{2}) =
\sum_{n=1}^3 \tau\|\mathbf{U}_i\|_1 \\
&+\sum_{n=1}^3{\langle \mathbf{M}_n, \nabla_n\X-\mathbf{U}_n\mathbf{V}_n^T\rangle + \frac{\mu_{n}}{2}\|\nabla_n\X-\mathbf{U}_n\mathbf{V}_n^T\|^2_F}\\
&+\langle \mathbf{\mathbf{\Gamma}}_1,y-\mathbf{\Psi}\left(\Z\right) \rangle+\frac{\mu_4}{2}\|y-\mathbf{\Psi}\left(\Z\right)\|^2_F +\frac{1}{2} \|\mathbf{E}\|_2\\
&+ \langle \mathbf{\mathbf{\Gamma}}_2,\Z-\X-\E \rangle+\frac{\mu_5}{2}\|\Z-\X-\E\|^2_F,
\end{split}
\end{equation}
\end{shrinkeq}

where $\mathbf{M}_n$, $n = 1,2,3$, $\mathbf{\mathbf{\Gamma}}_1$ and $\mathbf{\mathbf{\Gamma}}_2$ are the Lagrange multipliers and $\mu_n$s are a positive scalar in ADMM algorithm.
We then alternatively optimize each variable involved in (\ref{cs_lag}) with all others fixed, by following steps:

\textbf{Update} $\Z$. Extracting all items containing $\Z$ from Eq. (\ref{cs_lag}), we can obtain that:
\begin{equation}
\begin{split}\label{fun_x_cs}
&\Z:=\mathop{\argmin}_{\Z} \frac{\mu_4}{2}\left\|\mathbf{\Psi}\left(\Z\right)-\left(y+\frac{\mathbf{\mathbf{\Gamma}}_1}{\mu_4}\right) \right\|_F^2+\\
&\qquad \frac{\mu_5}{2}\left\|\Z-\left(\X+\E+\frac{\mathbf{\mathbf{\Gamma}}_2}{\mu_5}\right) \right\|_F^2.\\
\end{split}
\end{equation}

Optimizing (\ref{fun_x_cs}) with respect to $\Z$ can thus be equivalently treated as solving the following linear system:
\begin{equation}
\label{solver_x_cs}
\begin{split}
&\mu_4\mathbf{\Psi}^*\mathbf{\Psi}(\Z) +\mu_5\Z=\\
& \ \ \ \ \ \ \mu_5(\X+\E)+\mu_4\mathbf{\Psi}^*y+\mathbf{\Psi}^*\mathbf{\mathbf{\Gamma}}_1-\mathbf{\mathbf{\Gamma}}_2,\\
\end{split}
\end{equation}
where $\mathbf{\Psi}^* $ indicates the adjoint of $\mathbf{\Psi} $. This linear system can be solved by off-the-shelf techniques such as the preconditioned conjugate gradient method.

\textbf{Update} $\X$. Extracting all items containing $\X$ from Eq. (\ref{cs_lag}), similar to the process of solving $\Z$, we can obtain that optimizing  $\X$ is equivalent to solve the following linear system:
\begin{equation}
\begin{split}\label{fun_x1_cs}
& \sum_{n=1}^3\mu_{n}\mathbf{D}^*\mathbf{D}(\X) +\mu_5\X=\\
&\sum_{n=1}^3\mu_{n}\mathbf{D}^*\left(\mathbf{U}_{n}\mathbf{V}_{n}^{T}+\frac{\mathbf{M}_{n}}{\mu_{n}}\right)+\mu_{5}\left(\Z\!-\!\E\!+\!\frac{\mathbf{\mathbf{\Gamma}}_{2}}{\mu_{5}}\right).\\
\end{split}
\end{equation}

Similar to the process of solving $\mathbf{X}$ in the GCS denosing model, we can get that $\mathbf{X}$ can be efficiently computed by
\begin{equation}
\small
\label{slover_x1_cs}
\left\{
\begin{split}
&\mathbf{H}_x=\mu_{5}(\Z-\E)+\mathbf{\mathbf{\Gamma}}_{2}\! +\!\sum_{n=1}^3 \mathbf{D}_n^*(\mu_{n}\mathbf{U}_n\mathbf{V}_n^{T}+\mathbf{M}_n),\\
&\mathbf{T}_x=\mu_{1}|\text{fftn}(\mathbf{D}_1)|^2+\mu_{2}|\text{fftn}(\mathbf{D}_2)|^2+\mu_{3}|\text{fftn}(\mathbf{D}_3)|^2, \\
&\X=\text{ifftn}\left(\frac{\text{fftn}(\mathbf{\mbox{Fold}(H_x)})}{\mu\mathbf{1}+\mu\mathbf{T}_x}\right),
\end{split}
\right.
\end{equation}
where fftn and ifftn indicate fast Fourier transform and its inverse transform, respectively.

\textbf{Update} $\mathbf{U}_n,n=1,2,3$. Extracting all items containing $\mathbf{U}_n$ from Eq. (\ref{cs_lag}), similar to the process of solving $\mathbf{U}_{n}$ in the GCS denosing model, we get that $\mathbf{U}_{n}$ can be efficiently computed by
\begin{equation}
\label{slover_u_cs}
\mathbf{U}_n=\mathcal{S}_{\frac{\tau}{\mu_{n}}}\left(\left(\nabla_n\X+\mathbf{M}_n/\mu_{n}\right)\mathbf{V}_n\right).
\end{equation}

\textbf{Update} $\mathbf{V}_n,n=1,2,3$. Extracting all items containing orthogonal $\mathbf{V}_n$ from Eq. (\ref{cs_lag}), similar to the process of solving $\mathbf{V}_{n}$ in the GCS denosing model, we get that $\mathbf{V}_{n}$ can be efficiently computed by
\begin{equation}
\label{solver_v_cs}
\left\{
\begin{split}
& [\mathbf{B},\mathbf{D},\mathbf{C}]= \text{svd}((\nabla_n\X+\mathbf{M}_n/\mu)^T\mathbf{U}_n),\\
& \mathbf{V}_n=\mathbf{B}\mathbf{C}^T.
\end{split}
\right.
\end{equation}

\textbf{Update} $\E$. Extracting all items containing $\E$ from Eq.(\ref{cs_lag}),  similar to the process of solving $\E$ in the GCS denosing model, we can get that $\E$ can be efficiently computed by

\begin{equation}
\label{slover_e_cs}
\E=\frac{\mu_{5}(\Z-\X)+\mathbf{M}_{5}}{1+\mu_{5}}.
\end{equation}

Based on the general ADMM principle, the multipliers are further updated by the following equations:
\begin{equation}
\label{solver_m_cs}
\small
\left\{
\begin{split}
&\mathbf{M}_n=\mathbf{M}_n+\mu_{n}\left(\nabla_n\X-\mathbf{U}_n\mathbf{V}_n^{T} \right),n=1,2,3,\\
&\mathbf{\mathbf{\Gamma}}_1=\mathbf{\mathbf{\Gamma}}_1+\mu_{4}\left(y-\mathbf{\Psi}(\Z)\right),\\
&\mathbf{\mathbf{\Gamma}}_2=\mathbf{\mathbf{\Gamma}}_2+\mu_{5}\left(\Z-\X-\E\right).\\
\end{split}
\right.
\end{equation}

The above steps can then be summarized as $\mathbf{Algorithm}$ $\mathbf{2}$, for solving the model (\ref{cs_model}).
%, i.e., $\mathbf{Algorithm}$ $\mathbf{1}$, for solving the model (\ref{function_model}).
\begin{algorithm}
\caption{The E-3DTV for HSI compressed sensing reconstruct algorithm}\label{alg_cs}
\small
\begin{algorithmic}[1]
\renewcommand{\algorithmicrequire}{\textbf{Input:}}
\renewcommand{\algorithmicensure}{\textbf{End}}
\REQUIRE The compressive measures $y\in\mathbb{R}^{m},m\ll h\times w\times s$,  TV unfolding matrix rank: $r$, Regularized parameters $\tau$, Stopping criteria $\epsilon_1$,$\epsilon_2$.

\renewcommand{\algorithmicrequire}{\textbf{Initialization:}}
\renewcommand{\algorithmicensure}{\textbf{End}}
\REQUIRE Initial $\Z$, $\X$, $\E$, $\textbf{U}_n$, $\textbf{V}_n$, $\mathbf{M}_{n}$, $\mathbf{\mathbf{\mathbf{\Gamma}}}_{1}$, $\mathbf{\mathbf{\mathbf{\Gamma}}}_{2}$.

\WHILE {not converge}
\STATE
 Update $\Z$, $\X$, $\E$  by Eq.(\ref{solver_x_cs}), Eq.(\ref{slover_x1_cs})  and (\ref{slover_e_cs}), respectively.
\STATE
Update $\mathbf{U}_{n}$, $\mathbf{V}_{n}$ by Eq.(\ref{slover_u_cs}) and (\ref{solver_v_cs}), respectively .
\STATE
Update the multipliers $\mathbf{M}_{n}$, $\mathbf{\mathbf{\Gamma}}$ by Eq. (\ref{solver_m_cs}).
\STATE
Check the convergence conditions\\
$\quad \Vert y-\mathbf{\Psi}(\mathbf{X})\Vert_2^F\leq\epsilon_1$\\
$\quad \Vert\nabla_{n}\Z-\mathbf{U}_n\mathbf{V}_n^T\Vert_2^F/\Vert\mathbf{Y}\Vert_2^F\leq\epsilon_2,n=1,2,3$\\
\ENDWHILE

\renewcommand{\algorithmicrequire}{\textbf{Output:}}
\renewcommand{\algorithmicensure}{\textbf{End}}
\REQUIRE  The reconstructed HSI image $\Z\in\mathbb{R}^{w\times h\times s}$.
\end{algorithmic}
\end{algorithm}

%% ===========================================================================
\begin{table*}[!]
\footnotesize
\caption{Quantitative comparison of all competing methods under different sampling ratio on four datasets.}
\centering
\begin{tabular}{ c c c c c c c c c }
\Xhline{3\arrayrulewidth}
Data & Sampling & Quality &\multicolumn{6}{c}{Methods}\\
\cline{4-9}
set & ratio & indices &SpaRCS & KCS & JTTV & SLNTCS & JTenRe3DTV &ours\\
\Xhline{2\arrayrulewidth}
\multirow{15}{*}{Urban} &\multirow{3}{*}{0.3\%} &PSNR &18.04 &15.01 &18.90 &17.62 &23.48 &\textbf{24.17} \\
& & SSIM &0.2974 &0.1121 &0.3730 &0.1552 &0.3388 &\textbf{0.5881} \\
& & ERGAS &462.37 &693.67 &419.99 &485.40 &255.74 &\textbf{230.26} \\
\cline{2-9}
&\multirow{3}{*}{1\%} &PSNR &18.33 &18.51 &19.96 &20.42 &27.93 &\textbf{28.86} \\
& & SSIM &0.3119 &0.1982 &0.4489 &0.2348 &0.6049 &\textbf{0.7976} \\
& & ERGAS &441.76 &515.53 &373.06 &351.99 &158.96 &\textbf{141.51} \\
\cline{2-9}
&\multirow{3}{*}{5\%} &PSNR &19.87 &20.54 &26.64 &24.98 &35.18 &\textbf{38.10} \\
& & SSIM &0.3789 &0.4139 &0.7775 &0.5326 &0.8305 &\textbf{0.9506} \\
& & ERGAS &399.97 &370.58 &182.05 &219.86 &77.48 &\textbf{72.98} \\
\cline{2-9}
&\multirow{3}{*}{10\%} &PSNR &23.84 &23.45 &34.22 &27.62 &36.99 &\textbf{41.60} \\
& & SSIM &0.5619 &0.5118 &0.9266 &0.6002 &0.8724 &\textbf{0.9688} \\
& & ERGAS &298.62 &272.78 &91.50 &159.94 &65.27 &\textbf{64.83} \\
\cline{2-9}
&\multirow{3}{*}{20\%} &PSNR &30.45 &29.94 &42.35 &30.60 &38.04 &\textbf{43.13} \\
& & SSIM &0.7356 &0.6648 &\textbf{0.9745} &0.7322 &0.8894 &0.9727 \\
& & ERGAS &118.62 &124.77 &\textbf{54.40} &118.46 &57.54 &62.06 \\
\Xhline{2\arrayrulewidth}
\multirow{15}{*}{DC Mall} &\multirow{3}{*}{0.3\%} &PSNR &18.46 &16.24 &19.75 &17.83 &22.10 &\textbf{24.46} \\
& & SSIM &0.2228 &0.047 &0.3066 &0.1188 &0.3205 &\textbf{0.6171} \\
& & ERGAS &467.33 &635.11 &408.50 &518.46 &324.00 &\textbf{244.14} \\
\cline{2-9}
&\multirow{3}{*}{1\%} &PSNR &18.70 &19.48 &20.86 &20.66 &28.90 &\textbf{30.10} \\
& & SSIM &0.2409 &0.1423 &0.4129 &0.2708 &0.7897 &\textbf{0.8689} \\
& & ERGAS &456.27 &450.13 &360.06 &373.86 &142.82 &\textbf{128.59} \\
\cline{2-9}
&\multirow{3}{*}{5\%} &PSNR &20.40 &19.68 &27.14 &25.40 &38.92 &\textbf{41.32} \\
& & SSIM &0.4177 &0.4724 &0.7908 &0.5752 &0.9673 &\textbf{0.9871} \\
& & ERGAS &430.98 &522.79 &185.48 &220.35 &47.26 &\textbf{37.12} \\
\cline{2-9}
&\multirow{3}{*}{10\%} &PSNR &24.63 &19.94 &34.52 &27.92 &40.12 &\textbf{44.46} \\
& & SSIM &0.6095 &0.4939 &0.9479 &0.7387 &0.9746 &\textbf{0.9936} \\
& & ERGAS &347.45 &467.29 &84.84 &165.81 &41.56 &\textbf{26.38} \\
\cline{2-9}
&\multirow{3}{*}{20\%} &PSNR &33.65 &30.94 &44.65 &34.33 &41.15 &\textbf{47.32} \\
& & SSIM &0.9065 &0.8564 &0.9940 &0.9043 &0.9792 &\textbf{0.9966} \\
& & ERGAS &82.31 &118.52 &25.79 &85.68 &36.94 &\textbf{19.08} \\
\Xhline{2\arrayrulewidth}
\multirow{15}{*}{Moffett Field} &\multirow{3}{*}{0.3\%} &PSNR &26.58 &23.71 &27.66 &24.85 &29.78 &\textbf{30.17} \\
& & SSIM &0.5592 &0.001 &0.5975 &0.0333 &0.2869 &\textbf{0.6995} \\
& & ERGAS &264.89 &368.24 &237.38 &331.08 &\textbf{180.28} &182.70\\
\cline{2-9}
&\multirow{3}{*}{1\%} &PSNR &27.13 &26.03 &29.01 &27.03 &34.01&\textbf{34.19} \\
& & SSIM &0.5691 &0.0166 &0.6598 &0.2666 &0.6198 &\textbf{0.8369} \\
& & ERGAS &256.62 &295.53 &206.65 &257.57&\textbf{107.48} &119.19 \\
\cline{2-9}
&\multirow{3}{*}{5\%} &PSNR &29.23 &28.07 &34.73 &30.31 &41.67 &\textbf{43.42} \\
& & SSIM &0.6528 &0.2492 &0.8559 &0.4194 &0.8975 &\textbf{0.9713} \\
& & ERGAS &234.04 &226.74 &114.02 &181.53 &49.37&\textbf{47.02} \\
\cline{2-9}
&\multirow{3}{*}{10\%} &PSNR &33.36 &29.93 &39.45 &32.18 &43.18 &\textbf{47.72} \\
& & SSIM &0.7374 &0.4292 &0.9447 &0.4301 &0.9223 &\textbf{0.9833} \\
& & ERGAS &210.29 &217.25 &65.12 &178.19 &42.65 &\textbf{36.39} \\
\cline{2-9}
&\multirow{3}{*}{20\%} &PSNR &39.11 &34.35&46.85 &37.93 &44.12&\textbf{49.77} \\
& & SSIM &0.8577 &0.6157 &0.9861 &0.8197 &0.9333 &\textbf{0.9868} \\
& & ERGAS &61.08 &116.60 &\textbf{32.10} &68.19 &38.50 &33.16 \\
\Xhline{2\arrayrulewidth}
\multirow{15}{*}{Lowal Altitude} &\multirow{3}{*}{0.3\%} &PSNR &16.91 &13.75 &17.69 &16.67 &24.47 &\textbf{24.50}\\
& & SSIM &0.2687 &0.0237&0.3412 &0.0678 &0.4818 &\textbf{0.6357} \\
& & ERGAS &383.13 &490.69 &349.63 &399.84 &\textbf{146.30} &148.59\\
\cline{2-9}
&\multirow{3}{*}{1\%} &PSNR &17.23 &16.85 &18.96 &19.14 &28.47&\textbf{30.76} \\
& & SSIM &0.2821 &0.1328 &0.4341 &0.3220 &0.7344 &\textbf{0.8623} \\
& & ERGAS &374.25 &387.59 &304.41 &272.33&98.85 &\textbf{74.90} \\
\cline{2-9}
&\multirow{3}{*}{5\%} &PSNR &17.96 &20.28 &25.32 &24.28 &35.39 &\textbf{38.32} \\
& & SSIM &0.3525 &0.4824 &0.7534&0.6136 &0.8983 &\textbf{0.9526} \\
& & ERGAS &322.52 &250.12 &158.30 &151.82 &50.03&\textbf{47.82} \\
\cline{2-9}
&\multirow{3}{*}{10\%} &PSNR &21.24&26.99 &32.97&31.56 &36.91 &\textbf{41.24} \\
& & SSIM &05326 &0.6522 &0.9151 &0.8385 &0.9242 &\textbf{0.6490}\\
& & ERGAS &231.86 &112.80 &73.35 &76.41 &43.64 &\textbf{41.77} \\
\cline{2-9}
&\multirow{3}{*}{20\%} &PSNR &31.60 &30.18&42.47 &33.88 &38.03&\textbf{43.05} \\
& & SSIM &0.8517 &0.7891 &\textbf{0.9724} &0.8927 &0.9389 &0.9693 \\
& & ERGAS &104.22&81.92 &\textbf{35.34} &55.43 &38.34 &42.82 \\
\Xhline{3\arrayrulewidth}
\end{tabular}
\label{table:assement3}
\end{table*}
%% ====================================================
\begin{figure*}[!]
\centering
\includegraphics[scale=0.5]{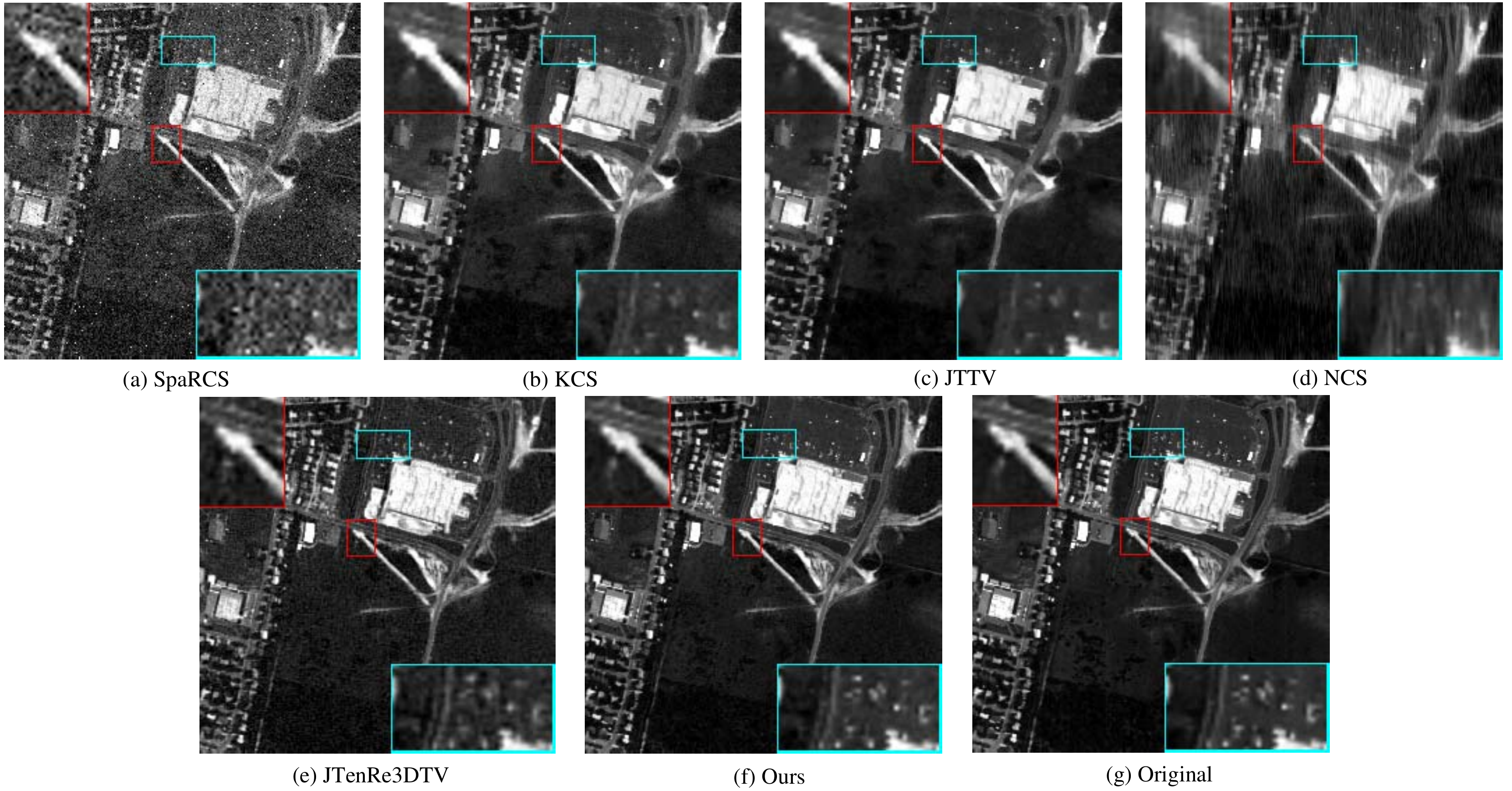}
\caption{The restoration performance of all competing methods on the band 5 of real Urban HSI.}
\label{fig_urban_cs_5}
\end{figure*}

\begin{figure*}[!]
\centering
\includegraphics[scale=0.5]{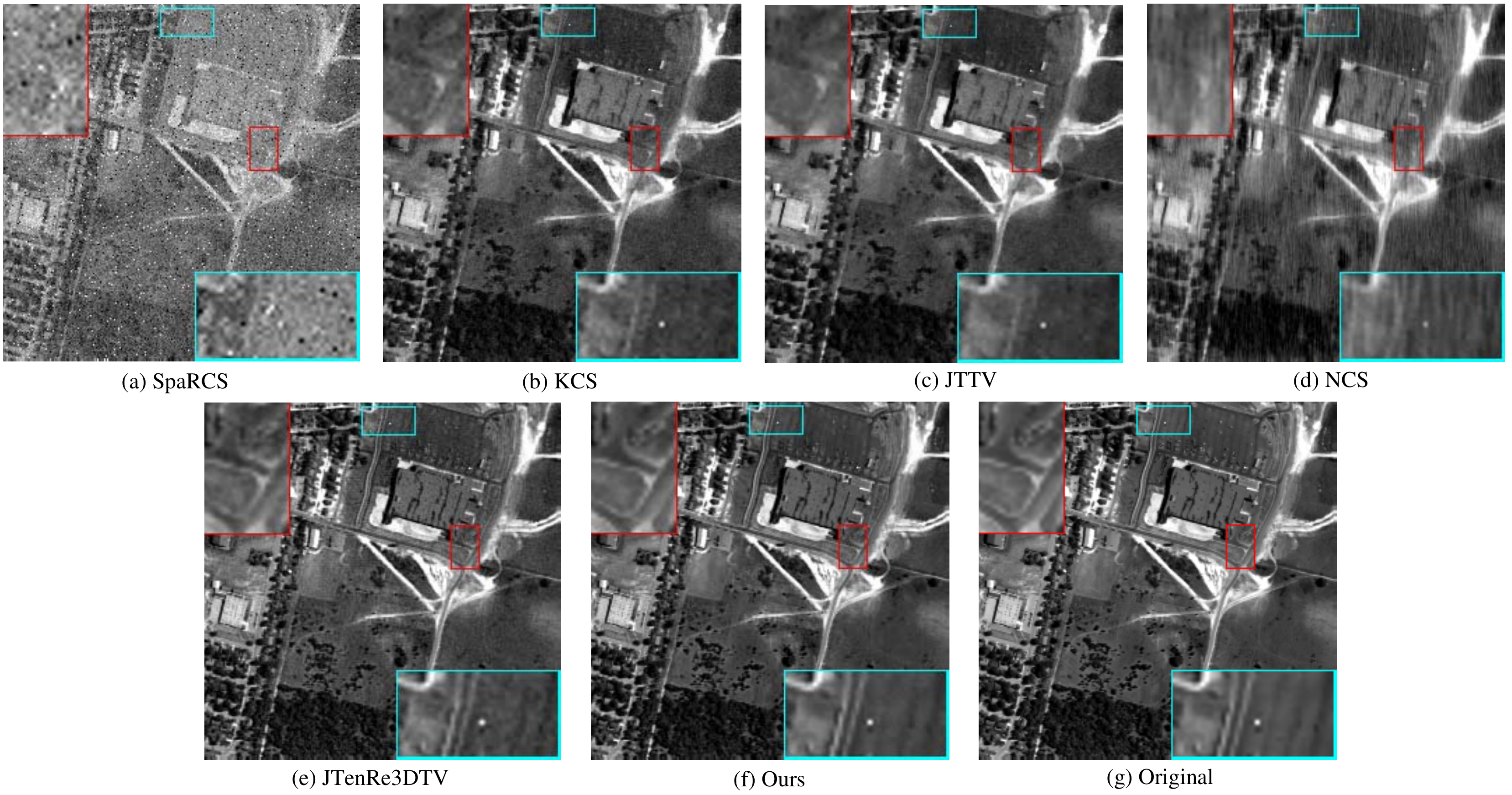}
\caption{The restoration performance of all competing methods on the band 160 of real Urban HSI.}
\label{fig_urban_cs_160}
\end{figure*}

\begin{figure*}[!]
\centering
\includegraphics[scale=0.5]{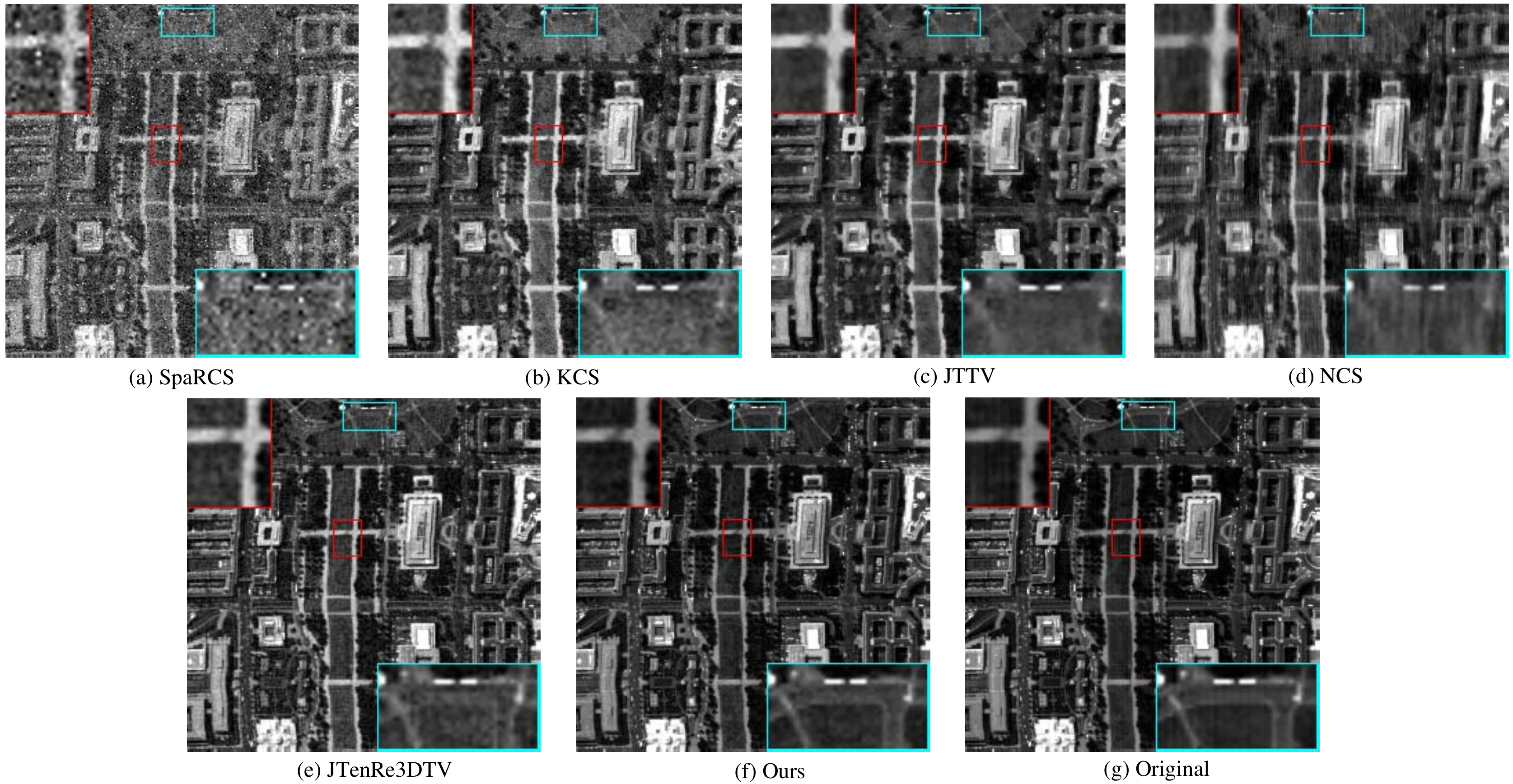}
\caption{The restoration performance of all competing methods on the band 5 of real DC mall HSI.}
\label{fig_dc_cs_5}
\end{figure*}

\begin{figure*}[!]
\centering
\includegraphics[scale=0.5]{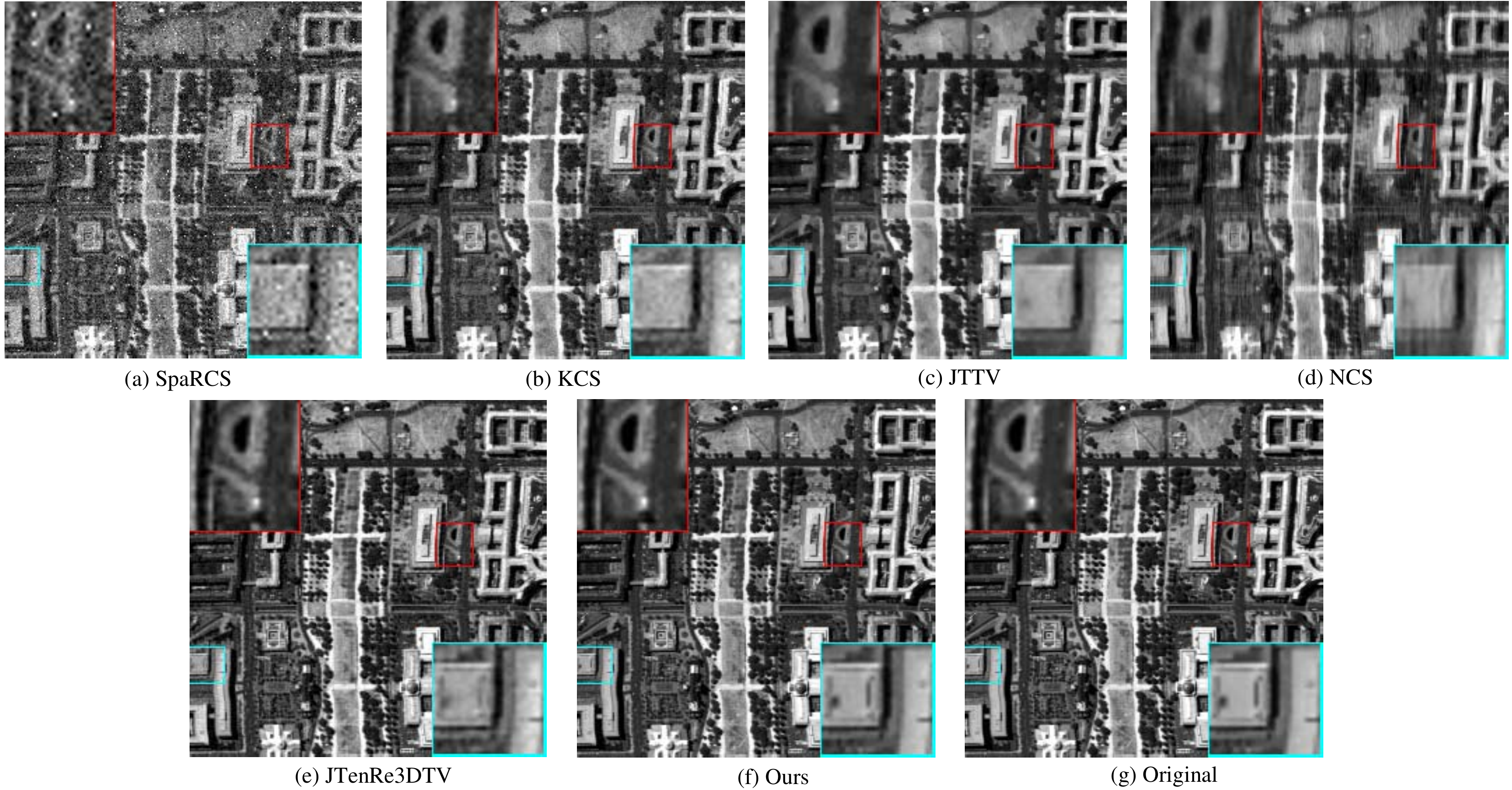}
\caption{The restoration performance of all competing methods on the band 160 of real DC mall HSI.}
\label{fig_dc_cs_160}
\end{figure*}

\begin{figure*}[!]
\centering
\includegraphics[scale=0.5]{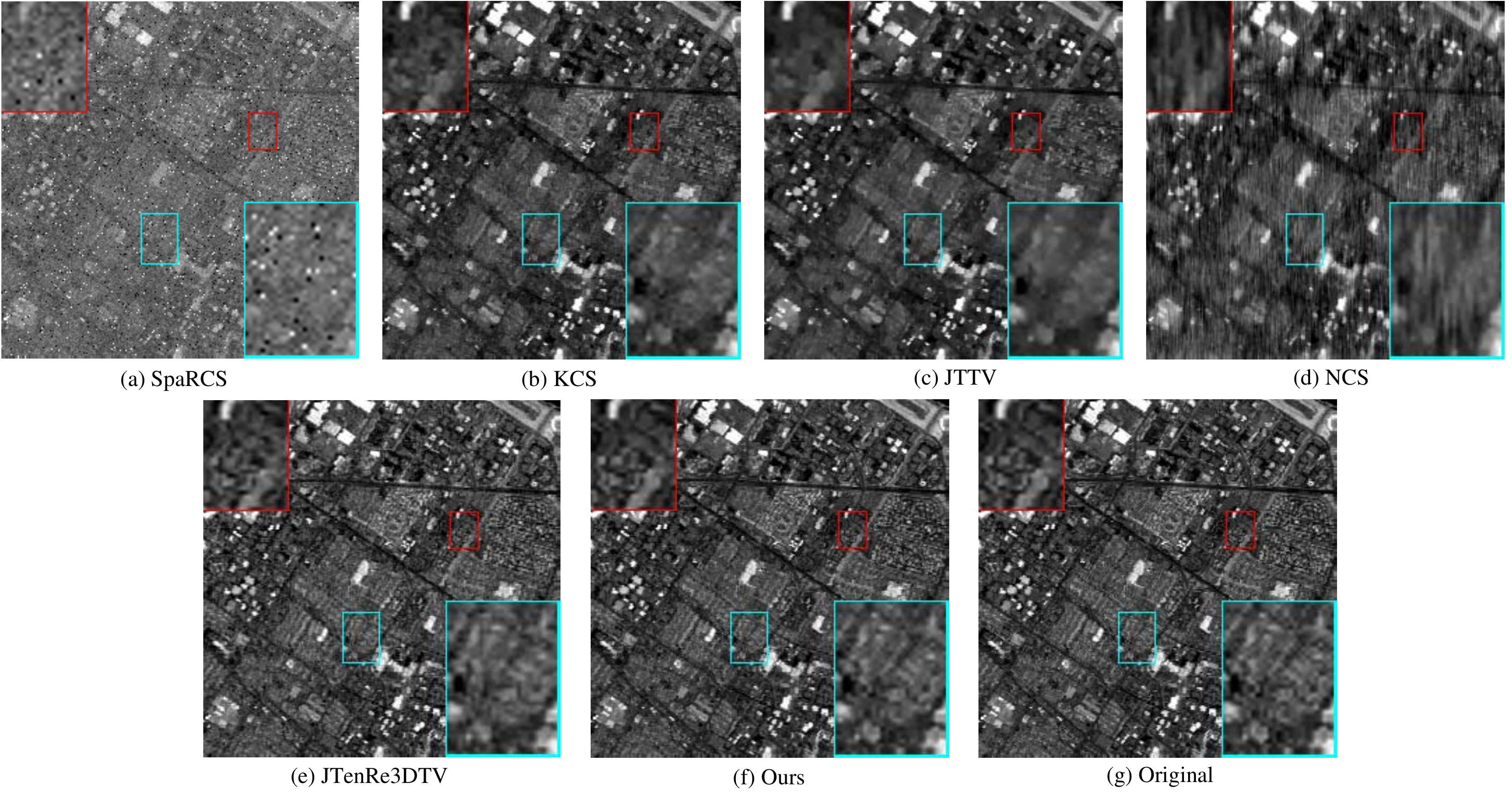}
\caption{The restoration performance of all competing methods on the band 80 of real moffett HSI.}
\label{fig_moffett_cs_80}
\end{figure*}

\begin{figure*}[!]
\centering
\includegraphics[scale=0.5]{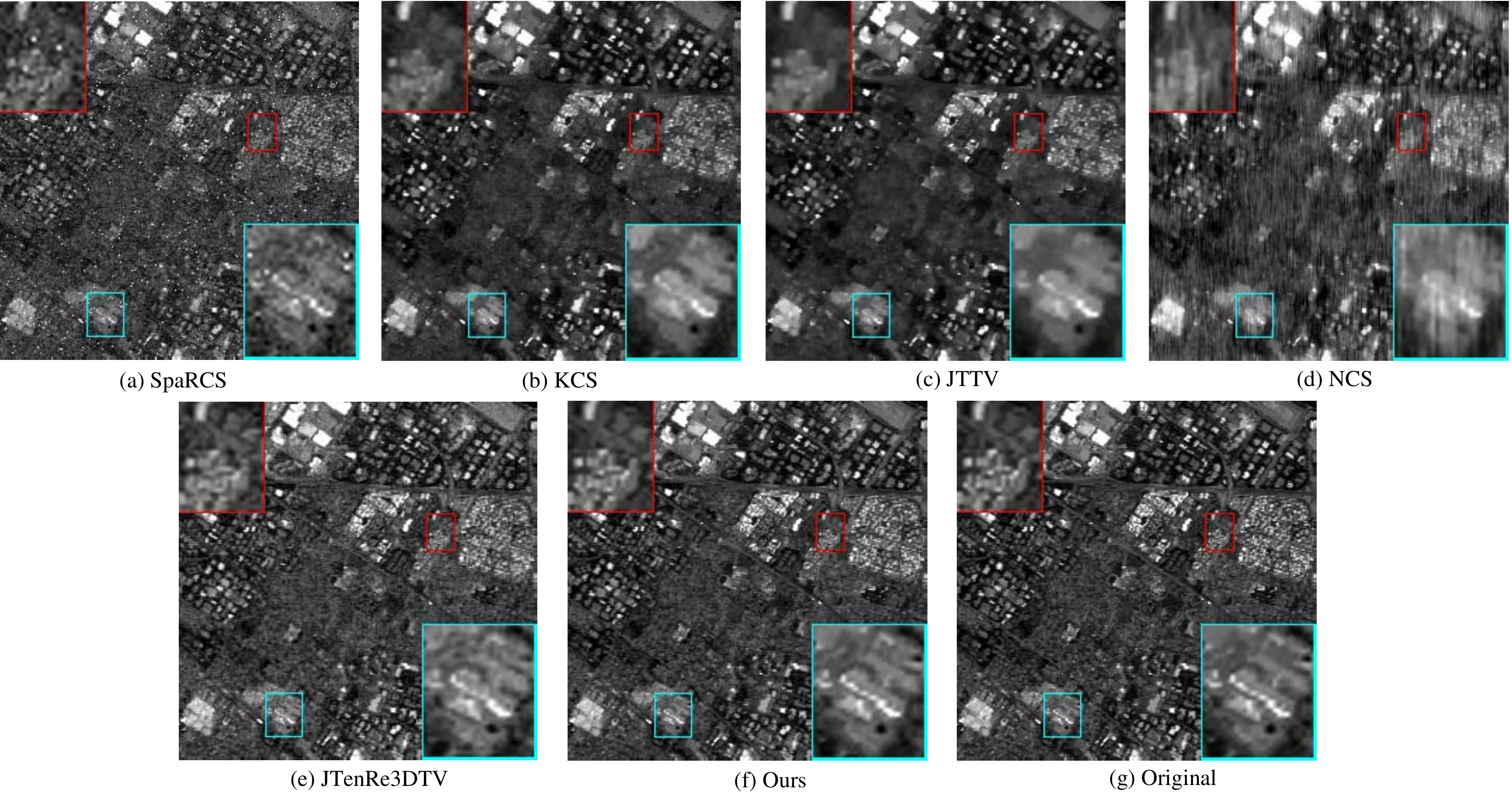}
\caption{The restoration performance of all competing methods on the band 160 of real moffett  HSI. }
\label{fig_moffett_cs_160}
\end{figure*}

\begin{figure*}[!]
\centering
\includegraphics[scale=0.5]{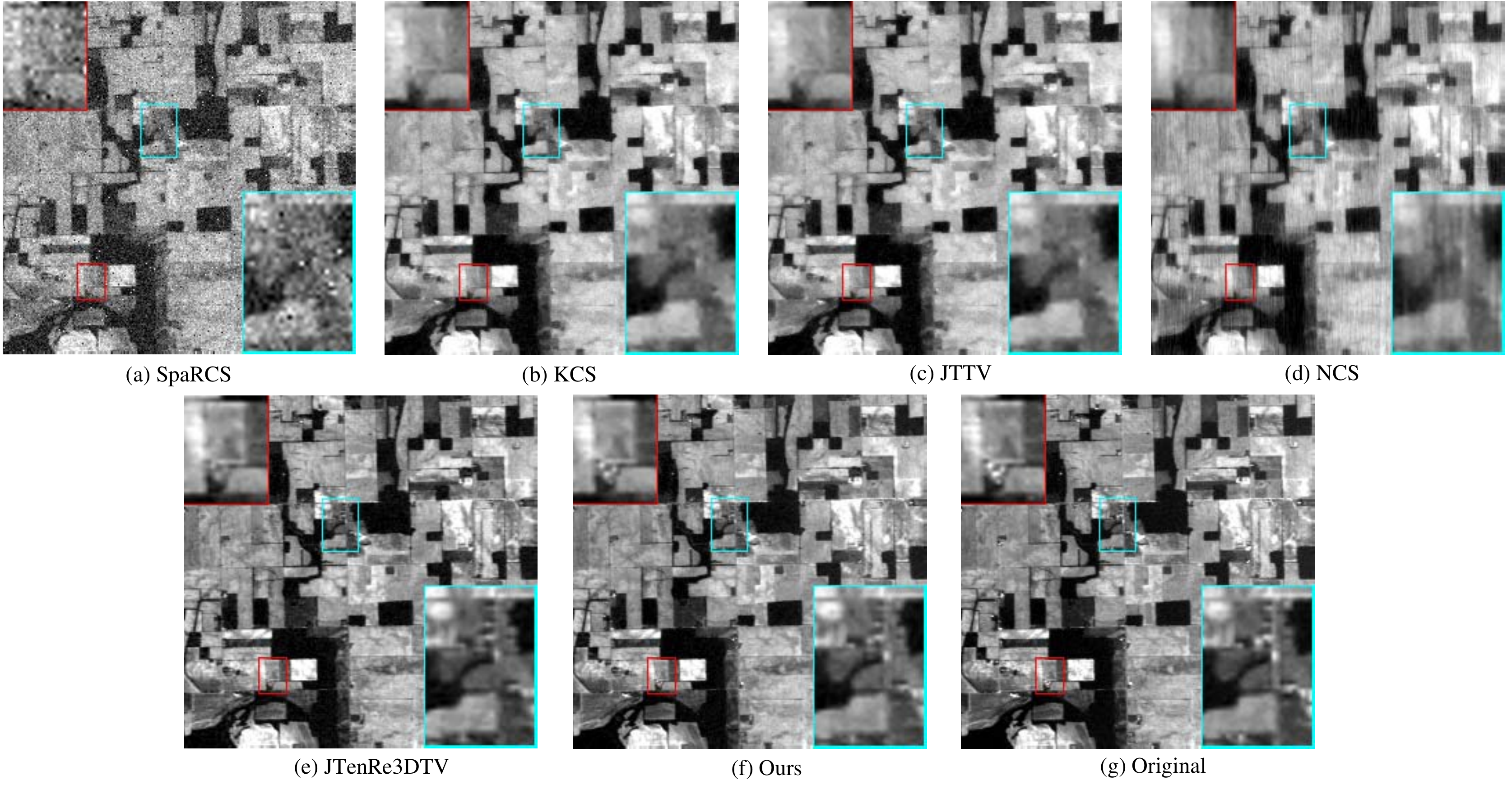}
\caption{The restoration performance of all competing methods on the band 5 of real lowal HSI.}
\label{fig_lowal_cs_5}
\end{figure*}

\begin{figure*}[!]
\centering
\includegraphics[scale=0.5]{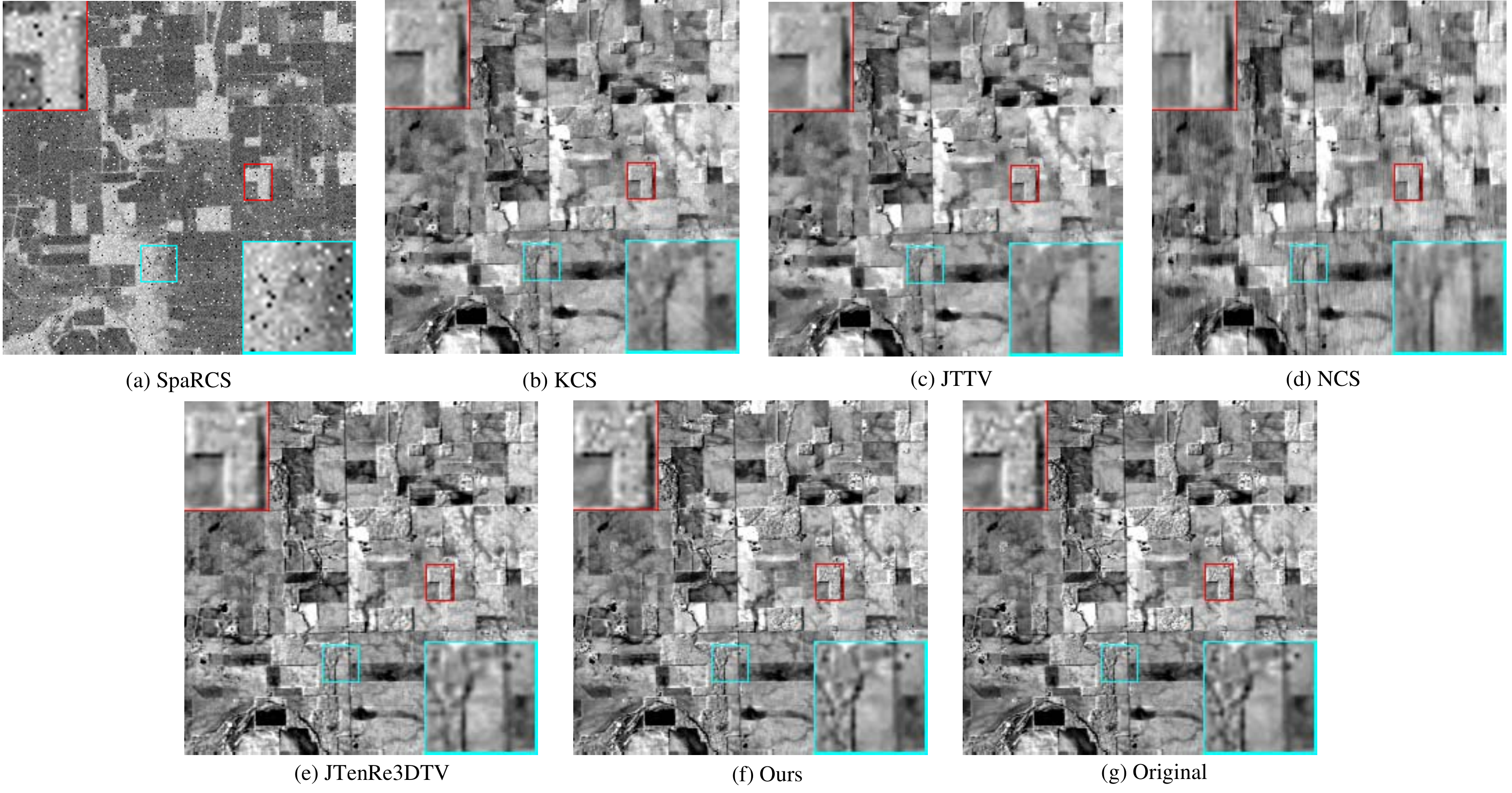}
\caption{The restoration performance of all competing methods on the band 160 of real lowal HSI.}
\label{fig_lowal_cs_80}
\end{figure*}

\section{Experimental Results on HSI Compressed Sensing}
In this section, several real HSI data experiments were carried out to substantiate the effectiveness of the E-3DTV regularizer on HSI compressed sensing.

Four popular real-world HSI data sets were used in our experiments, including the HYDICE Urban,  the HYDICE Washington DC Mall, the AVIRIS Moffett Field, and the AVIRIS Lowal altitude data sets.
After removing the seriously polluted bands and cropping images for each data set, the HSI cube used for the experiments are of sizes $256 \times  256 \times 188$, $256 \times  256 \times 191$ ,$256 \times  256 \times 200$ and $256 \times  256 \times 200$, respectively.

To thoroughly evaluate the performance of the proposed method, we implemented five state-of-the-art methods for comparison, including
Kronecker compressed sensing (KCS) \cite{Li2012A}, low-rank and joint-sparse recovery (SparCS) \cite{Golbabaee2012Hyperspectral}, joint nuclear/TV norm minimization (JTTV) \cite{Golbabaee2013Joint}, self-learning tensor nonlinear compressed sensing (SL-TNCS) \cite{Yang2015Compressive} and Joint tensor/reweight 3DTV norm minimization (JTenRe3-DTV) \cite{Wang2017Compressive}.
As previously stated, we adopted the random permuted Hadamard transform as the compressive operator. Five different sampling ratios, namely $0.3\%$, $1\%$, $5\%$, $10\%$, $20\%$ were considered. To sufficiently assess the performance of all compared approaches, we also employed the above three quantitative picture quality indices (PQIs), PSNR, SSIM and ERGAS, for performance evaluation.

As for parameter settings, there are two parameters in the model that need to be set: the desired rank $r$ and the sparsity basis coefficient $\tau$. Note that with the increase of the sampling rate, the ideal $r$ and $\tau$ would also increase. Specifically, we set the value of rank $r$ as 4 or 5, and the sparsity basis coefficient $\tau$ as 0.0075 at the sampling rate are $0.3\%$, $1\%$. With the increase of the sampling rate, the small value of rank and sparsity basis coefficient will cause the algorithm to lose some information, so as the sampling rates are $5\%$, $10\%$, $20\%$, we increase the rank $r$ as a value from $7$ to $10$, and set the sparsity basis coefficient $\tau$ as 0.015. As for the competing methods' parameters, we use the default settings or tuned all their values to possibly guarantee their optimal performance for fair comparison.

Table. \ref{table:assement3} compares the reconstruction results by all the compared methods. It can be easily seen that, the proposed method achieves a significantly improved performance under all such five sampling ratios with all PQIs, as compared with other competing methods.

In terms of visual quality, two representative bands of four HSIs in sampling ratio $5\%$ were presented in Fig. \ref{fig_urban_cs_5}-\ref{fig_lowal_cs_80}. Some demarcated areas in the subfigure are enlarged in the right bottom corner and left top corner for better visualization. From these figures and the quantitative indices numerical table, we can see that the method utilizing the local smoothness property of image, such as the JTTV, SL-NCS, JTenRe3DTV, can relatively more effectively eliminate the non-smoothness caused by compression reconstruction as the restoration image of SpaRCS show. Meanwhile, it is evident that the method based on the E-3DTV regularizer can preserve the best detail edge, texture information and image fidelity than other competing method.

\begin{figure*}[!]
\centering
\includegraphics[scale=0.48]{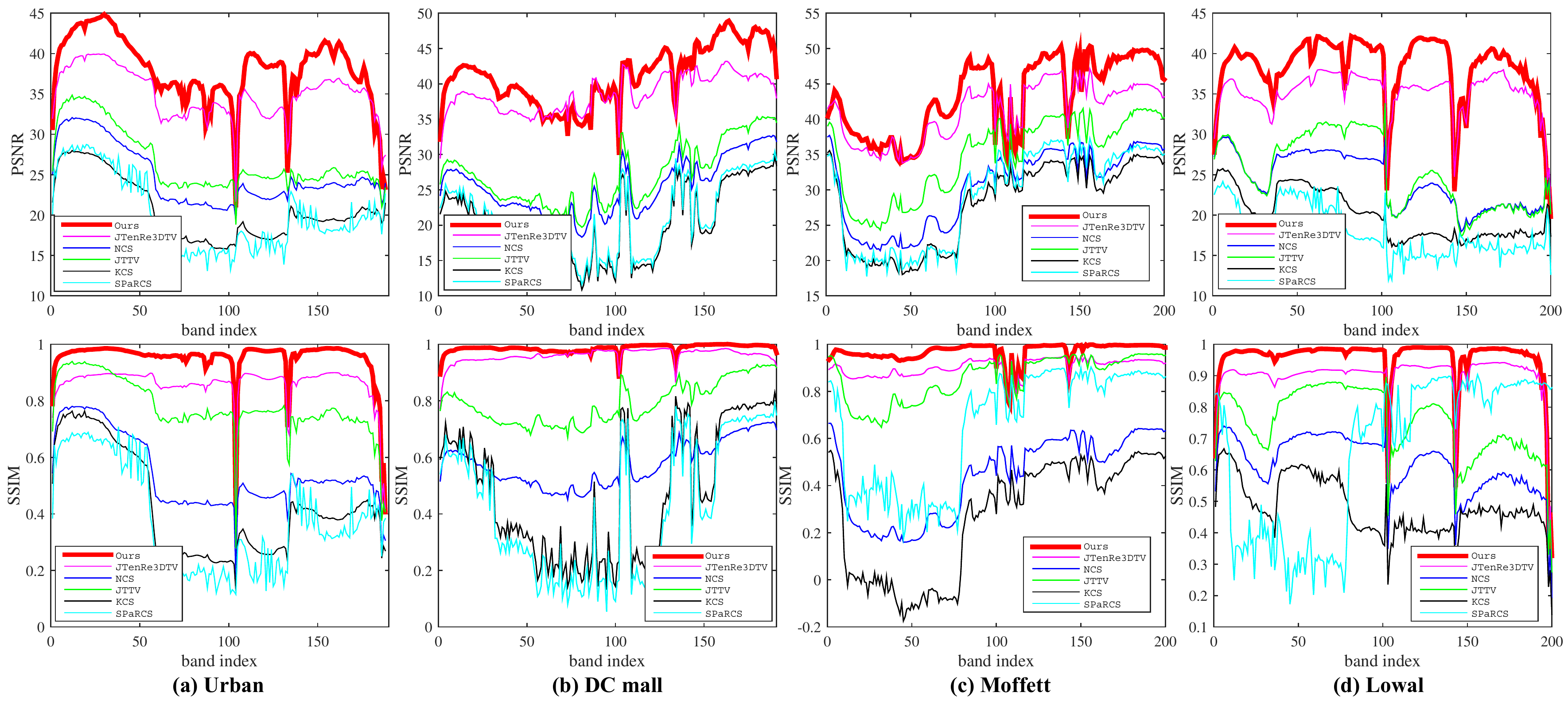}
\caption{ The quantitative comparisons of different methods for each band of four data sets under the sampling ratio $5\%$ in terms of PSNR and SSIM. The column from the first to the fourth correspond to the results obtained on Urban, Moffett Field, DC Mall, and Lowal Altitude, respectively.}
\label{cs_indices}
\end{figure*}

We also show the PSNR and SSIM values by reconstructing each band of three data sets under sampling ratio $5\%$ in Fig. \ref{cs_indices}. It is easy to see that the proposed method can get higher SSIM and PSNR values than other ones for almost all HSI bands. This comprehensively illustrates the robustness of our proposed method.

From the above presentation, we can see that the structures of these four real HSIs are different, while our proposed method can get satisfactory performance. Combined the performance of denoising, it should be expected to use the E-3DTV regularizer as a general tool to ameliorate the performance for general HSI processing tasks.

%% ===========================================================================
\section{Conclusion}

In this paper, we have introduced an enhanced 3DTV regularization (E-3DTV, briefly), to better reflect the sparsity characteristic of the gradient maps of a natural HSI. It is surprising to us that by simply embedding such regularization term as the unique regularizer into the models designed for HSI denoising and compressed sensing, respectively, it attains state-of-the-art performance superior to the methods previously used for the tasks. This makes it hopeful further extend such a regularizer to more general HSI processing tasks, which is the main task of our future research.

There are still some limitations on the practical use of the proposed E-3DTV regularizer. For example, it should be necessary to consider to design an automatic parameter tuning strategy based on certain HSI as well as its size and structure complexity. Besides, the E-3DTV regularizer is still imposed on the unfolding matrices of the considered HSI, while not its multi-dimensional structure (i.e., tensor) itself. It should be another meaningful investigation on reformulating the proposed regularizer directly on the tensor to more sufficiently utilizing the structural knowledge of the HSI. We will further investigate these issues in our future research.

\section*{Appendix}

%\section{abstract}
In this supplementary material, we provide the proof of the equivalence between problem (10) and problem (11) in the maintext.
%% ===========================================================================
\subsection*{Proof of the equivalence}

We call two problems equivalent if from a solution of one, a solution
of the other is readily found, and vice versa.
We can then give the following theorem:
\begin{theorem}
\label{theorem_e}
For any $\G\in\mathbb{R}^{hw\times s}$, the problem
\begin{equation}
\begin{split}\label{CSC_new}
&\min_{\V  \in \mathbb{R}^{s\times r}} \| \G \V  \|_1 \\
&{ \rm s.t.}~ \|\G \V \|_F=\|\G \|_F, \mathbf{V} ^T\mathbf{V} =\mathbf{I},~
\end{split}
\end{equation}
is equivalent to
\begin{equation}
\begin{split}\label{CSC}
&\min_{\U ,\V }\|\mathbf{U} \|_1 \\
&{\rm s.t.}~ \G =\mathbf{U} \mathbf{V} ^T,~ \mathbf{V} ^T\mathbf{V} =\mathbf{I},\\&~~~~~~
\mathbf{U}  \in\mathbb{R}^{hw\times r}, ~\mathbf{V}  \in\mathbb{R}^{s\times r}.\\
\end{split}
\end{equation}
\end{theorem}
\begin{proof} \textbf{(a)}.
  We first prove that for any  $[\U, \V]$ satisfies the constrains of (\ref{CSC}),  $\V $ satisfies the constrains of (\ref{CSC_new}).

  By $\G = \U \V^{T}$ and $\V^{T}\V =\mathbf{I}$, we can obtain that,
  \begin{equation}\label{eq1}
  \begin{split}
  &\|\G\V \|_F^2  = \|\U \V^{T}\V \|_F^2 = \|\U \|_F^2 \\
  =&\mbox{tr}\left(\U^{T}\U \right) =  \mbox{tr}\left(\U^{*T}\U \V^{T}\V \right) \\
  =& \mbox{tr}\left(\V \U^{T}\U \V^{T}\right) \\
  =& \|\U \V^{T}\|_F^2 \\
  =& \|\G\|_F^2.
  \end{split}
  \end{equation}
  This implies that $\V $ satisfies the constrains of (\ref{CSC_new}).

\textbf{(b)}.
We then prove that for any $\V $ satisfies the constrains of (\ref{CSC_new}), by letting $\U=\G\V$,  $[\U, \V]$ satisfies the constrains of (\ref{CSC}).

By $\|\G \V \|_F=\|\G \|_F$ and $\V^{T}\V =\mathbf{I}$, we can obtain that,
  \begin{equation}\label{eq2}
  \begin{split}
  &\|\G-\U\V^T \|_F^2 = \|\G-\G\V\V^T \|_F^2   \\
  =& \mbox{tr}\left( \G^T\G \right) - 2\mbox{tr}\left( \G^T\G\V\V^T  \right) + \mbox{tr}\left( \V\V^T\G^T\G\V\V^T \right)\\
  =& \mbox{tr}\left( \G^T\G \right) - 2\mbox{tr}\left( \V^T\G^T\G\V  \right) + \mbox{tr}\left( \V^T\G^T\G\V \right)\\
  =&\mbox{tr}\left( \G^T\G \right) - \mbox{tr}\left( \V^T\G^T\G\V  \right)\\
  =&\|\G \|_F - \|\G \V \|_F\\
  =&0.
  \end{split}
  \end{equation}
  Thus we have $\G=\U\V^T$, which implies $[\U, \V]$ satisfies the constrains of (\ref{CSC}).

  Now, it is easy to prove that for any $\V^*$ solves (\ref{CSC_new}), then $[\G\V^*,\V^*]$ solves (\ref{CSC}).  If $[\G\V^*,\V^*]$ is not a solution of (\ref{CSC}), then there is an $[\U^+,\V^+]$ satisfies the constrains of  (\ref{CSC}), and  $\|\G\V^+\|_1=\|\U^+\|_1<\|\G\V^*\|_1$. By \textbf{(a)}, we know that $\V^+$ is a satisfies the constrains of (\ref{CSC_new}), this is contradict to the assumption that  $\V^*$ solves (\ref{CSC_new}). Similarly, we can prove that if  $[\U^*,\V^*]$ solves (\ref{CSC}), then $\V^*$ solves (\ref{CSC_new}). This finishes the proof.
  $\blacksquare$
\end{proof}

By \textbf{Theorem 2} we know that the  problem (10) and problem (11) in the maintext is equivalent to each other.
%\end{appendix}

\end{document}